\documentclass[preprint,12pt,number,sort&compress]{elsarticle}
\usepackage{amssymb}
\usepackage{rotating}
\usepackage{amsmath,amsfonts}
\usepackage{amsthm}
\theoremstyle{definition}
\newtheorem{definition}{Definition}[section] 

\theoremstyle{plain}
\newtheorem{assumption}{Assumption}[section] 

\theoremstyle{plain}
\newtheorem{proposition}{Proposition}[section] 
 
\newtheorem{lemma}[proposition]{Lemma}
\newtheorem{corollary}[proposition]{Corollary}

\theoremstyle{remark}
\newtheorem{remark}{Remark}[section]
\usepackage{subcaption}
\usepackage{xcolor}
\usepackage{hyperref}
\usepackage{graphicx}
\usepackage{changepage}
\usepackage{comment}
\usepackage{tabularx}
\usepackage{booktabs}
\usepackage{float}
\begin{document}
	
\begin{frontmatter}
\title{A Unified Framework for Uncertainty-Aware Explainable Artificial Intelligence: A Case Study in Power Quality Disturbance Classification} %% Article title

\author{Yinsong Chen}
\ead{yinsong.chen@deakin.edu.au}
\affiliation{organization={School of Engineering, Deakin University},%Department and Organization
	city={Melbourne},
	postcode={3216}, 
	state={VIC},
	country={Australia}}

\author{Samson S. Yu\corref{1}}
\ead{samson.yu@deakin.edu.au}

\affiliation{organization={School of Engineering, Deakin University},%Department and Organization 
	city={Melbourne},
	postcode={3216}, 
	state={VIC},
	country={Australia}}
	
\author{Zhong Li} %% Author name
\ead{zhong.li@fernuni-hagen.de}

%% Author affiliation
\affiliation{organization={Faculty of Mathematics and Computer Science, FernUniversität in Hagen},%Department and Organization
	postcode={58084}, 
	country={Germany}}

\author{Chee Peng Lim} %% Author name
\ead{cplim@swin.edu.au}

%% Author affiliation
\affiliation{organization={Department of Computing Technologies, Swinburne University of Technology},%Department and Organization 
	city={Hawthorn},
	postcode={3122}, 
	state={VIC},
	country={Australia}}

\cortext[1]{Corresponding author}
%% Abstract
\begin{abstract}
Post-hoc explainable AI (XAI) methods usually return one attribution map, even when the model represents uncertainty in its parameters. We define the \emph{explanation distribution} as the distribution of attribution maps obtained from sampled models. The uncertainty-aware relevance attribution operator (UA-RAO) summarises this distribution using the mean, dispersion, quantiles, and agreement sets. The theory separates posterior-approximation error from finite-sample error and accounts for changes across activation boundaries and for stochastic explainers. On a 15-class power-quality-disturbance benchmark, the mean occlusion explanation from a deep ensemble aligns better with known disturbance regions than the deterministic baseline, although the improvement depends on the disturbance type. Tests with controlled input distortions show that additive noise changes the explanations more than amplitude scaling or aligned temporal shifts.
\end{abstract}

%% Keywords
\end{frontmatter}

\section{Introduction}

Deep learning has improved classification performance in safety-critical domains, such as medical imaging, fault diagnosis, and energy systems, by learning rich representations directly from raw data \cite{ali2023explainable,samek2019explainable}. Yet the internal decision logic of deep neural networks (DNNs) remains opaque to end users, creating what is commonly termed the \emph{black-box problem} \cite{ali2023explainable}. This opacity limits deployment in safety-critical settings where predictions must be audited, contested, or verified against domain knowledge. Regulatory frameworks such as the European Union's General Data Protection Regulation \cite{gdpr2016} further require that automated decisions be interpretable, intensifying demand for principled transparency mechanisms. Critically, such mechanisms must also account for the uncertainty inherent in model predictions: an explanation that is itself unreliable offers limited assurance to practitioners who must act on it.

Explainable artificial intelligence (XAI) addresses this by producing post-hoc attribution maps that identify which input features drive a model's prediction \cite{ding2022explainability,raza2025translime}. Methods in this class include occlusion sensitivity, Gradient-weighted Class Activation Mapping (Grad-CAM), Local Interpretable Model-agnostic Explanations (LIME), Layer-wise Relevance Propagation (LRP), and SHapley Additive exPlanations (SHAP) \cite{samek2019explainable}. These methods are now well taxonomised \cite{ali2023explainable,nauta2023anecdotal} and widely deployed across scientific and engineering domains. However, virtually all are deterministic: they return a single explanation for a fixed model and fixed input. This design implicitly treats the model as a point estimate, an assumption that is invalid when the model represents inherent predictive uncertainty.

Bayesian neural networks (BNNs) address prediction uncertainty at its source by maintaining a distribution over model parameters, enabling principled separation of epistemic and aleatoric uncertainty \cite{chen5100162interplay,luo2025complexity,yang2025enhancing}. This distributional view is particularly valuable in the presence of noisy inputs, distribution shift, or ambiguous composite patterns. This introduces a new challenge for XAI: each posterior sample induces a different attribution map, making the explanation itself a random variable. Applying a deterministic XAI method to a single posterior sample, or to a point estimate of a BNN, discards this variability entirely, potentially reporting high-confidence explanations for fundamentally uncertain predictions \cite{bykov2020much,mulye2025uncertainty}. Despite their joint relevance to trustworthy machine learning, no existing framework systematically characterises the explanation itself as a statistical object whose distribution can be formally defined and bounded \cite{wood2024model}.

Several targeted approaches have begun to address this gap. Some works adapt a specific XAI operator to the Bayesian setting: Bykov et al.\ \cite{bykov2020much} extended LRP to BNNs (B-LRP) to propagate posterior uncertainty into saliency maps, Peltola \cite{peltola2018local} adapted LIME to Bayesian predictive distributions via Kullback--Leibler projections, and Wood et al.\ \cite{wood2024model} introduced entropy-based permutation importance measures for uncertainty-aware models. Others treat the explanation as a distributional object: Bykov et al.\ \cite{bykov2021explaining} formalised the notion of an explanation distribution induced by BNN posterior sampling, Clare et al.\ \cite{clare2022explainable} showed that distributional attribution substantially improves trustworthiness in geophysical prediction, and Mulye and Valdenegro-Toro \cite{mulye2025uncertainty} demonstrated that gradient-based explanations exhibit substantial variance across ensemble members, undermining reliance on point-estimate attributions. Each of these contributions, however, addresses a specific method--model pair or application domain. Recent work has also reduced the cost of Bayesian uncertainty estimation at deployment. Linearised Laplace approximations \cite{immer2021improving,daxberger2021laplace} produce closed-form Gaussian predictive distributions from one trained network, while sparse variational Gaussian-process posteriors \cite{titsias2009variational} represent function-space uncertainty using a small set of inducing variables. These methods reduce the cost of posterior sampling and are therefore useful for real-time monitoring in settings such as power grids. However, they do not specify how parameter or function uncertainty should be carried into post-hoc explanations. Our framework can use any approximation that defines a probability measure over model parameters, including these efficient posterior approximations. The question this paper addresses is therefore: given an arbitrary BNN posterior approximation and an arbitrary attribution operator, how can one formally define, characterise, and summarise the resulting distribution of explanations within a single probabilistic framework?

Power quality disturbance (PQD) classification provides a compelling case study for such a framework. PQD events (e.g., sags, swells, harmonics, flicker, and transients) arise from nonlinear loads, power-electronic devices, and growing renewable integration \cite{khan2023xpqrs}, and can trigger failures in sensitive equipment and undermine grid reliability. Deep convolutional neural networks (CNNs) now dominate this task, with accuracy exceeding 98\% on standard benchmarks \cite{wang2019novel}, trained on large synthetic datasets spanning 15 disturbance categories \cite{machlev2021open}. However, safety-critical operation demands more than accuracy: predictions must carry calibrated confidence, particularly for composite or out-of-distribution disturbances, and explanations must be verifiable against known physical signal characteristics \cite{machlev2021measuring}. The combination of a well-defined 15-class taxonomy, established synthetic benchmarks, and the simultaneous need for reliable uncertainty and interpretable attribution makes PQD classification an ideal validation domain.

This paper presents a unified framework for representing and summarising uncertainty in post-hoc explanations. The framework is evaluated on the 15-class Expert Power Quality Recognition System (XPQRS) benchmark \cite{khan2023xpqrs} using three model-distribution methods and three attribution operators. The main contributions are:

\begin{itemize}
    \item A measure-theoretic framework that defines the explanation distribution by applying any measurable deterministic or stochastic attribution operator to a model distribution, without modifying either component.
    \item The uncertainty-aware relevance attribution operator (UA-RAO), a general operator family converting an explanation distribution into interpretable uncertainty-aware attribution maps via multiple summary statistics, including mean, variance, coefficient of variation, quantiles, and agreement sets.
    \item Finite-sample and asymptotic results that separate explanation-space approximation bias from Monte Carlo error. Further results cover boundary crossings, model and explainer variability, aligned input sensitivity, and explanation-band coverage under exchangeability.
    \item An evaluation on a 15-class synthetic PQD benchmark in which deep-ensemble mean occlusion improves localisation in paired comparisons, although the improvement varies by class. The experiments also report input-sensitivity and finite-sample self-coverage results.
\end{itemize}

The remainder of this paper is organised as follows. Section~\ref{sec:pre} introduces the necessary background, covering the PQD classification problem, Bayesian neural network formulations, and the XAI operators considered. Section~\ref{sec:theo} presents the unified framework, formalising explanation distributions as push-forward measures and introducing the UA-RAO together with its theoretical guarantees. Section~\ref{sec:eval} describes the evaluation framework, including ground-truth disturbance segmentation, point-explanation localisation metrics, and distributional assessment criteria. Section~\ref{sec:experiments} reports experimental results on the 15-class synthetic PQD benchmark and includes a qualitative assessment on real-world measured data. Section~\ref{sec:conclusion} concludes the paper.

\section{Preliminaries}
\label{sec:pre}
This section establishes the prerequisites for the framework in Section~\ref{sec:theo}. The three subsections introduce, respectively, the PQD classification setting, the Bayesian neural network formulation with approximate inference, and the local relevance attribution operator—the three components that the framework combines.
\subsection{Power Quality Disturbance Classification}
\label{sec:pqd}
Let $x[n] \in \mathbb{R}$ denote a discrete-time voltage waveform of length $N$ ($n=1,\dots,N$) sampled from a power network operating at nominal frequency. The undisturbed reference signal is modelled as a pure sinusoid,
\begin{equation}
	\label{eq:1}
	x_0[n] = A\sin(\omega n+\phi), n=1,\dots,N,
\end{equation}
where $A$, $\omega$, and $\phi$ represent the amplitude, angular frequency, and phase, respectively. A \emph{simple disturbance} refers to a local deviation of $x$ from $x_0$ attributable to a single physical event (e.g., sag, transient). A \emph{complex disturbance} arises from the superposition or concatenation of $k>1$ simple events. In this work, 15 representative PQD categories are considered, covering both simple and complex disturbance types, as listed in \cite{wang2019novel}.

\noindent\textbf{Dataset Formulation.} Each raw waveform $x[n]$ with fixed length $N$ is treated as a univariate time series. Let $\mathcal{X} \subseteq \mathbb{R}^{N}$ denote the input space and $\mathcal{Y} = \{1,2,\dots,M\}$ the label space. A labelled PQD dataset can be represented by \begin{equation}
\mathcal{D} = \bigl\{(x^{(i)},\, y^{(i)})\bigr\}_{i=1}^{N_{\mathcal{D}}}, \quad x^{(i)} \in \mathcal{X},\; y^{(i)} \in \mathcal{Y},\end{equation} where $N_{\mathcal{D}}$ is the number of samples in the dataset.

\noindent\textbf{Deep Learning Classifier.} The PQD classifier is a parametric function
\begin{equation}
	\label{eq:classifier}
	f_{\theta}: \mathcal{X} \rightarrow \Delta^{M-1},
\end{equation}
where $\Delta^{M-1} = \{p \in \mathbb{R}^{M} : p_m \geq 0,\; \sum_m p_m = 1\}$ is the $(M{-}1)$-simplex. The output $f_{\theta}(x) \in \Delta^{M-1}$ is a probability vector over the $M$ classes, and the predicted label is $\hat{y} = \arg\max_{m}\, [f_{\theta}(x)]_m$. The parameters $\theta$ are estimated by minimising the categorical cross-entropy loss over $\mathcal{D}$,
\begin{equation}
	\label{eq:loss}
	\mathcal{L}(\theta) = -\frac{1}{|\mathcal{D}|}
	\sum_{(x,y) \in \mathcal{D}}
	\sum_{m=1}^{M} \mathbf{1}[y = m]\,\log [f_{\theta}(x)]_m,
\end{equation}
with $\mathbf{1}[\cdot]$ the indicator function. A one-dimensional convolutional neural network (1D-CNN) is adopted as the backbone, as its hierarchical feature extraction is well suited to the time-localised and morphologically distinct signatures of PQDs.
\subsection{Bayesian Neural Network}
Deterministic classifiers $f_{\theta}$ provide point predictions without quantifying predictive confidence. In safety-critical applications such as power quality monitoring, distinguishing reliable predictions from uncertain ones is essential. To this end, the classifier is formulated as a BNN by imposing a prior distribution $p(\theta)$ over the model parameters. Let $\Theta$ denote the parameter space, with $\theta\in\Theta$. Given the training set $\mathcal{D}$, Bayesian inference yields the posterior
\begin{equation}
	p(\theta \mid \mathcal{D})
	= \frac{p(\mathcal{D} \mid \theta)\,p(\theta)}
	{p(\mathcal{D})},
\end{equation} where \begin{equation}
	p(\mathcal{D} \mid \theta) = \prod_{(x,y) \in \mathcal{D}} p(y \mid x, \theta)
\end{equation} is the discriminative likelihood, and 
\begin{equation}
	p(\mathcal{D}) = \int_{\Theta} p(\mathcal{D}\mid\theta)\,p(\theta)\,\mathrm{d}\theta
\end{equation} is the evidence. For a new input $x$, the posterior predictive distribution, which captures epistemic uncertainty, is obtained by marginalising over the parameter posterior,
\begin{equation}
	p(y \mid x, \mathcal{D})
	= \int_{\Theta}
	p(y \mid x, \theta)\,
	p(\theta \mid \mathcal{D})\,
	\mathrm{d}\theta.
\end{equation}
\noindent\textbf{Approximate Posterior Inference} Since the exact posterior $p(\theta \mid \mathcal{D})$ is generally intractable for DNNs, an approximate posterior $q_{\phi}(\theta)$, parameterised by $\phi$, is introduced \cite{chen5100162interplay} such that
\begin{equation}
	q_{\phi}(\theta) \approx p(\theta|\mathcal{D}).
\end{equation}
The predictive distribution is correspondingly approximated as
\begin{equation}
	p(y|x,\mathcal{D}) \approx \int_{\Theta} p(y|x,\theta)q_{\phi}(\theta)d\theta. 
\end{equation}
Since this integral is generally intractable, it is estimated using Monte Carlo integration. Given $S$ samples $\theta^{(s)} \sim q_{\phi}(\theta)$, $s=1,...,S$,
\begin{equation}
	p(y|x,\mathcal{D}) \approx \frac{1}{S} \sum_{s=1}^S p(y|x,\theta^{(s)}).
\end{equation} By the law of large numbers, this estimator converges asymptotically to the predictive distribution,
\begin{equation}
	\lim_{S\rightarrow\infty} \frac{1}{S} \sum_{s=1}^S p(y|x,\theta^{(s)}) = p(y|x, \mathcal{D}).
\end{equation}
This Monte Carlo estimator is the computational basis on which the explanation distribution in Section~\ref{sec:theo} is made tractable.

\begin{remark}[Implicit Model Distribution of Deep Ensembles]
\label{rem:ensemble}
Deep ensembles do not specify an explicit posterior density. To place them within the proposed framework, let $Z=(\theta_0,U)\sim\mu_Z$ denote the random training state on $\mathcal Z$, including the initialisation and all stochastic optimisation choices. For a measurable training map $T_{\mathcal D}:\mathcal Z\to\Theta$, randomised training induces the model distribution
\begin{equation}
 q_\phi^{\mathrm{ens}}=(T_{\mathcal D})_\#\mu_Z.
\end{equation}
This distribution need not have a tractable density. If $Z_j$ are independent draws from $\mu_Z$, then $\theta_j=T_{\mathcal D}(Z_j)$ are i.i.d. and therefore exchangeable. For a finite ensemble, exchangeability remains a modelling assumption; fixed seed values alone do not guarantee it. Wilson and Izmailov \cite{wilson2020bayesian} discuss the multi-basin Bayesian-model-averaging view, while Izmailov et al.\ \cite{izmailov2021posteriors} compare deep ensembles with exhaustive Hamiltonian Monte Carlo (HMC) using predictive agreement and total variation.
\end{remark}
\subsection{Explainable Artificial Intelligence}
\label{sec:xai}
XAI methods are generally divided into \textit{global} and \textit{local} approaches. \textit{Global} methods characterise model behaviour over entire datasets, whereas \textit{local} methods attribute the relevance of individual predictions to specific input features through a relevance attribution operation \cite{bykov2021explaining}. This study focuses exclusively on \textit{local} XAI methods, which are well suited to PQD classification because they provide instance-specific interpretations essential for validating individual decisions in safety-critical power systems.

\noindent\textbf{Relevance Attribution Operator.} Let $f_{\theta}$ denote the classifier under parameter realisation $\theta$. A \textit{local} relevance attribution operator is defined as 
\begin{equation}
	\tau_x(\theta):=\mathcal{T}_{x,\theta}[f_{\theta}](x) = R_{\theta}(x),
\end{equation}
where $R_{\theta}(x)\in\mathbb{R}^{N}$ is a time-aligned saliency signal. The component $R_{\theta}(x)[n]$ quantifies the influence of $x[n]$ on the classifier decision under parameter realisation $\theta$. For an attribution operator with internal randomness, such as LIME perturbation sampling, write $\tau_x(\theta,\xi)$, where $\xi\in\Xi$ has distribution $P_\xi$ and is independent of $\theta$. The pair $(\theta,\xi)$ then replaces $\theta$ in the definitions below, and $q_\phi$ is replaced by the product measure $q_\phi\otimes P_\xi$ on $\Theta\times\Xi$; see \textbf{Remark~\ref{rem:lip_ops}}.

\section{A Unified Framework for Uncertainty-Aware Explainable Artificial Intelligence}
\label{sec:theo}
Section~\ref{sec:pre} established the framework prerequisites: the approximate posterior $q_\phi$ over network parameters, and the local relevance attribution map $\tau_x$ introduced in Section~\ref{sec:xai}. This section formalises the distribution these objects jointly induce over explanations, establishes its theoretical properties, and introduces the UA-RAO as a general summarisation operator.
\begingroup

\subsection{Explanation Distributions}

For a fixed input $x$, write $\tau_x(\theta)=\mathcal{T}_{x,\theta}[f_\theta](x)$. When the explainer has an internal random state $\Xi\sim P_\xi$, independent of $\Theta$, write $\tau_x(\theta,\xi)$ and use the product measure $q_\phi\otimes P_\xi$.

\begin{assumption}[Measurable attribution operator]
\label{ass:1}
For each fixed $x$, $\tau_x:\Theta\to\mathbb{R}^N$ is Borel measurable. For a stochastic explainer, $\tau_x:\Theta\times\Xi\to\mathbb{R}^N$ is jointly Borel measurable.
\end{assumption}

\begin{definition}[Explanation distribution]
\label{def:1}
The approximate explanation distribution at $x$ is the push-forward
\begin{equation}
 Q_x:=p_{q_\phi}(\cdot\mid x,\mathcal D)=(\tau_x)_\#q_\phi,
 \qquad
 Q_x(A)=q_\phi\!\left(\tau_x^{-1}(A)\right),
 \label{eq:12}
\end{equation}
for $A\in\mathcal B(\mathbb R^N)$, where $\mathcal B(\mathbb R^N)$ is the Borel $\sigma$-algebra. For a stochastic explainer, replace $q_\phi$ by $q_\phi\otimes P_\xi$. The corresponding ideal explanation distribution is
\[
 P_x=(\tau_x)_\#p(\theta\mid\mathcal D),
\]
with the same $P_\xi$ in the stochastic-explainer case.
\end{definition}

This definition applies to probability measures and does not require an explanation density. For fixed $x$, a finite ReLU network has finitely many Borel activation-pattern regions. Occlusion and Grad-CAM are measurable on these regions. For LIME, we assume that the implemented map is jointly measurable in the model parameters and sampled random state. Continuity of $f_\theta$ does not by itself imply continuity of a gradient- or optimisation-based explainer.

\begin{proposition}[Monte Carlo estimation]
\label{prop:mc}
Let Assumption~\ref{ass:1} hold, let $\gamma:\mathbb R^N\to\mathbb R^d$ be measurable, and suppose $\mathbb E_{Q_x}\|\gamma(R)\|_2<\infty$. If $R^{(1)},\ldots,R^{(S)}$ are i.i.d. draws from $Q_x$, then
\begin{equation}
 \frac1S\sum_{s=1}^S\gamma(R^{(s)})
 \xrightarrow{\mathrm{a.s.}}\mathbb E_{Q_x}[\gamma(R)].
 \label{eq:mc}
\end{equation}
For a stochastic explainer, the draws are generated from $q_\phi\otimes P_\xi$.
\end{proposition}

\begin{proof}
Measurable transformations preserve i.i.d. sampling, the push-forward identity gives $\mathbb E[\gamma(\tau_x(\Theta))]=\mathbb E_{Q_x}[\gamma(R)]$, and the finite-dimensional vector strong law applies \cite{billingsley2017probability,durrett2019probability}.
\end{proof}

\begin{assumption}[Bounded explanation range]
\label{ass:range}
For fixed $x$, $P_x$ and $Q_x$ are supported on a common measurable set $\mathcal R_x\subset\mathbb R^N$ with
\[
 D_R:=\sup_{r,r'\in\mathcal R_x}\|r-r'\|_2<\infty.
\]
\end{assumption}

This assumption provides the finite moments and bounded functional ranges used below. It holds automatically for occlusion maps based on bounded class-probability differences. For raw Grad-CAM, it must be assumed for explanations under both the ideal and approximate posteriors. It is separate from the min--max normalisation used only for visual display.

\begin{proposition}[Approximation bias and finite-sample error]
\label{prop:bias}
Suppose Assumption~\ref{ass:range} holds. Let $R^{(1)},\ldots,R^{(S)}\stackrel{\mathrm{i.i.d.}}{\sim}Q_x$, let $\gamma:\mathbb R^N\to\mathbb R$ be $1$-Lipschitz, and let $\delta\in(0,1)$. Then, with probability at least $1-\delta$,
\begin{equation}
 \left|\frac1S\sum_{s=1}^S\gamma(R^{(s)})-
 \mathbb E_{P_x}[\gamma(R)]\right|
 \le W_1(P_x,Q_x)+D_R\sqrt{\frac{\log(2/\delta)}{2S}}.
 \label{eq:bias}
\end{equation}
Here, $W_1$ denotes the $1$-Wasserstein distance on explanation space.
\end{proposition}

\begin{proof}
Add and subtract $\mathbb E_{Q_x}[\gamma(R)]$ and apply the triangle inequality. The dual form of $W_1$ bounds the approximation term by $W_1(P_x,Q_x)$ \cite{villani2003topics}. On $\mathcal R_x$, a $1$-Lipschitz scalar functional has range at most $D_R$, and Hoeffding's inequality gives the second term \cite{hoeffding1963probability}.
\end{proof}

\begin{corollary}[Simultaneous mean-map error]
\label{cor:meanmap}
Let $\bar R=S^{-1}\sum_{s=1}^S R^{(s)}$. If every explanation coordinate has range length at most $C_x$, then, with probability at least $1-\delta$,
\[
 \left\|\bar R-\mathbb E_{P_x}[R]\right\|_\infty
 \le W_1(P_x,Q_x)+C_x\sqrt{\frac{\log(2N/\delta)}{2S}}.
\]
For maps in $[0,1]^N$, one may take $C_x=1$.
\end{corollary}

The first term is the approximation bias in explanation space, and the second is the finite-sample Monte Carlo error. The bound applies to bounded $1$-Lipschitz scalar summaries and coordinatewise means. Variance, coefficient of variation, quantiles, and the localisation metrics in Section~\ref{sec:eval:point} require separate analysis. Direct estimation of $W_1(P_x,Q_x)$ also requires a suitable reference explanation distribution. The next lemma instead shows how local parameter changes and boundary crossings contribute to this distance.

\begin{lemma}[Local-stability bound with boundary-crossing penalty]
\label{lem:piecewise}
Let $G\subseteq\Theta\times\Theta$ be Borel and suppose directly that
\[
 \|\tau_x(\theta)-\tau_x(\theta')\|_2
 \le L_G\|\theta-\theta'\|_2,
 \qquad (\theta,\theta')\in G.
\]
Under Assumption~\ref{ass:range}, every coupling $\pi\in\Pi(p(\theta\mid\mathcal D),q_\phi)$ satisfies
\begin{equation}
 W_1(P_x,Q_x)\le
 L_G\int_G\|\theta-\theta'\|_2\,d\pi
 +D_R\,\pi(G^c).
 \label{eq:piecewise}
\end{equation}
Here, $\Pi(P,Q)$ denotes the set of couplings of probability measures $P$ and $Q$.
The bound may be tightened by minimising the right-hand side over these couplings.
\end{lemma}

\begin{proof}
Push $\pi$ through $(\tau_x,\tau_x)$, split the transportation cost over $G$ and $G^c$, use the assumed inequality on $G$, and use the explanation-range diameter on $G^c$.
\end{proof}

The set $G$ may represent a high-probability parameter region with compatible activation patterns. The second term accounts for activation-region crossings, tails, and other cases in which the local inequality does not hold. The pairwise inequality on $G$ is assumed directly and is not derived from smoothness or gradient bounds on a possibly non-convex region.

\begin{proposition}[Model and stochastic-explainer variability]
\label{prop:totalvar}
Let $\Theta\sim q_\phi$, $\Xi\sim P_\xi$, $\Theta\perp\Xi$, and $R=\tau_x(\Theta,\Xi)$ have finite second moments. Define $m_x(\theta)=\mathbb E_\Xi[\tau_x(\theta,\Xi)]$. Then, coordinatewise,
\begin{equation}
 \operatorname{Var}(R[n])=
 \underbrace{\operatorname{Var}_\Theta(m_x(\Theta)[n])}_{\text{model variability}}
 +\underbrace{\mathbb E_\Theta[\operatorname{Var}_\Xi(R[n]\mid\Theta)]}_{\text{explainer variability}}.
 \label{eq:totalvar}
\end{equation}
\end{proposition}

\begin{proof}
This is the law of total variance conditional on $\Theta$.
\end{proof}

For LIME, one independent $\xi$ is paired with each sampled $\theta$. This estimates total explanation variability but does not separate the two components in Eq.~\eqref{eq:totalvar}. We do not assume that the LIME/Lasso solution map is uniformly Lipschitz in the model parameters.

\begin{remark}[Applicability to the tested operators]
\label{rem:lip_ops}
Occlusion maps based on class-probability differences are measurable and bounded. Grad-CAM is Borel measurable as a piecewise-defined map and may cross activation boundaries. Lemma~\ref{lem:piecewise} applies when its within-$G$ pairwise inequality and bounded-range condition hold. Under joint measurability and the required moments, LIME on the product space is covered by Propositions~\ref{prop:mc} and~\ref{prop:totalvar}. Other discontinuous operators can also be used in the push-forward and Monte Carlo results when they are measurable, but a Wasserstein stability result requires an additional condition.
\end{remark}

\subsection{Uncertainty-Aware Relevance Attribution Operator}

The attribution target is a fixed class $c$. It may be specified by the user or chosen from the prediction, while the experiments below use the true class.

\begin{definition}[Uncertainty-Aware Relevance Attribution Operator]
\label{def:uarao}
Let $\mathcal P(\mathbb R^N)$ denote the probability measures on $\mathbb R^N$, and let $\mathcal P_k(\mathbb R^N)$ denote those with finite $k$th moment. For a chosen summary functional
\[
 \Phi:\mathcal D_\Phi\to\mathcal Y_\Phi,
 \qquad \mathcal D_\Phi\subseteq\mathcal P(\mathbb R^N),
\]
the UA-RAO is $\mathcal U_\Phi(x)=\Phi(Q_x)$. Its empirical version based on $S$ explanation draws is denoted $\widehat{\mathcal U}_\Phi(x)$. The domain is summary-specific: the mean requires $Q_x\in\mathcal P_1(\mathbb R^N)$, variance requires $Q_x\in\mathcal P_2(\mathbb R^N)$, coefficient of variation additionally uses a stabiliser $\kappa>0$, and coordinatewise quantiles and agreement frequencies require no moment condition. For the empirical summaries below, $\hat s_n$ is the sample standard deviation and $r_{(k),n}$ is the $k$th order statistic at coordinate $n$; $\alpha\in(0,1)$ is the quantile level, $\delta$ is a relevance threshold, and $\eta\in[0,1]$ is the required agreement fraction.
\end{definition}

\begin{table}[!htb]
\centering
\caption{Main UA-RAO summaries and their interpretation.}
\label{tab:uarao}
\footnotesize
\begin{tabularx}{\linewidth}{p{2.4cm}Xp{4.2cm}}
\toprule
Summary & Empirical form & Interpretation \\
\midrule
Mean & $\bar R=S^{-1}\sum_sR^{(s)}$ & Central tendency of member-wise explanations. \\
Variance & $\hat s_n^2=(S-1)^{-1}\sum_s(R^{(s)}[n]-\bar R[n])^2$ & Coordinatewise total dispersion. \\
$\alpha$-quantile & $r_{(\lceil\alpha S\rceil),n}$ & Coordinatewise lower-to-upper support summary. \\
Coefficient of variation & $\hat s_n/(|\bar R[n]|+\kappa)$ & Relative dispersion where mean relevance is non-negligible. \\
Agreement set & $\widehat{\mathcal E}_{\eta,\delta}=\{n:\widehat\rho_{n,S}^\delta\ge\eta\}$, $\widehat\rho_{n,S}^\delta=S^{-1}\sum_s\mathbf1\{R^{(s)}[n]>\delta\}$ & Features exceeding $\delta$ with agreement at least $\eta$. \\
\bottomrule
\end{tabularx}
\end{table}

\begin{remark}[Meaning and consistency of the summaries]
\label{rem:consistency}
The mean UA-RAO is the expected member-wise explanation. For a nonlinear operator such as Grad-CAM or LIME, it generally differs from the explanation of the posterior-predictive mean. The framework averages explanations and does not assume that explanation and model averaging can be interchanged.

The strong law gives consistency of the mean. Applying it to $R$ and $R^2$, followed by the continuous mapping theorem, gives consistency of the variance. The same argument applies to the coefficient of variation because its denominator is bounded below by $\kappa$. For coordinate $n$, let $F_n$ be the cumulative distribution function of $R[n]$. The empirical $\alpha$-quantile is consistent when the population quantile $q_{\alpha,n}$ is the unique $z$ satisfying
\[
 F_n(z^-)\le\alpha\le F_n(z),
\]
which allows atomic explanation distributions \cite{serfling1980approximation}. Finally, $\widehat\rho_{n,S}^\delta\to\rho_n^\delta=P(R[n]>\delta)$ almost surely, including when there is an atom at $\delta$. Because $N$ is finite, the thresholded set stabilises if the boundary-separation condition $\rho_n^\delta\ne\eta$ holds for every $n$.
\end{remark}

\begin{corollary}[Input stability under aligned coupling]
\label{cor:input}
Let $Q_{x,n}$ be the $n$th marginal of $Q_x$, let $d$ be a metric on the input space, and define
\[
 D_{\mathrm{coord}}(Q_x,Q_{x'})=\frac1N\sum_{n=1}^N W_1(Q_{x,n},Q_{x',n}).
\]
Let $A_{x,x'}$ be a fixed measurable alignment operator. If, for a measurable nonnegative function $L_{x,x'}$,
\[
 \frac1N\|\tau_x(\theta)-A_{x,x'}\tau_{x'}(\theta)\|_1
 \le L_{x,x'}(\theta)d(x,x'),
 \qquad \mathbb E_{q_\phi}[L_{x,x'}(\Theta)]<\infty,
\]
then
\begin{equation}
 D_{\mathrm{coord}}\!\left(Q_x,(A_{x,x'})_\#Q_{x'}\right)
 \le \mathbb E_{q_\phi}[L_{x,x'}(\Theta)]d(x,x').
 \label{eq:inputstab}
\end{equation}
For a stochastic explainer, the coupling reuses the same $(\Theta,\Xi)$.
\end{corollary}

\begin{proof}
Couple the two explanations using the same parameter draw, bound each marginal $W_1$ by the expected aligned coordinate difference, and sum over coordinates.
\end{proof}

The circular-shift alignment in Section~\ref{sec:sensitivity} is a coordinate permutation and therefore an $\ell_1$-isometry. Section~\ref{sec:sensitivity} uses this property to measure the remaining local distributional change after alignment.

\begin{proposition}[Finite-sample order-statistic explanation band]
\label{prop:orderband}
For a fixed coordinate $n$, suppose $R^{(1)}[n],\ldots,R^{(S+1)}[n]$ are exchangeable. Let $R_{(1),n}\le\cdots\le R_{(S),n}$ be the order statistics of the first $S$ values and let $1\le k<\ell\le S$. If ties occur with probability zero, then
\[
 \Pr\{R_{(k),n}\le R^{(S+1)}[n]\le R_{(\ell),n}\}
 =\frac{\ell-k}{S+1}.
\]
For a closed band with possible ties, the probability is at least $(\ell-k)/(S+1)$ \cite{david2003order}.
\end{proposition}

This result gives distribution-free, pointwise finite-sample coverage for a new exchangeable explanation draw. It describes coverage within the sampling mechanism rather than calibration against an ideal explanation distribution.

\endgroup

\section{Evaluation Framework}
\label{sec:eval}

Building on the theoretical guarantees in Section~\ref{sec:theo}, this section defines the evaluation framework for assessing UA-RAO explanations empirically. It specifies the ground-truth disturbance reference, localisation metrics for point explanations, their extension to distributional UA-RAO summaries, and the qualitative visualisation protocol. These definitions are independent of a particular posterior approximation or attribution operator and provide the basis for the experiments in Section~\ref{sec:experiments}.

\subsection{Ground-Truth Disturbance Segmentation}
\label{sec:eval:gt}

Quantitative evaluation of any XAI method requires a reference against which explanation quality can be measured. In PQD classification, ground-truth disturbance locations are directly available from the signal generation model, a property that avoids the annotation scarcity which limits XAI evaluation in many other domains \cite{nauta2023anecdotal}. In practice, the observed waveform $x$ is a noisy observation of a disturbed signal,
\begin{equation}
	x[n] = x_0[n] + d_x[n] + \sigma[n], \qquad n = 1, \dots, N,
\end{equation}
where $x_0[n]$ is the disturbance-free reference of Eq.~\eqref{eq:1}, $d_x[n]$ is the disturbance component associated with input $x$, and $\sigma[n]$ represents additive measurement noise. Since $d_x[n]$ is known from the signal generation process, disturbance regions are identified as time indices at which the disturbance component exceeds a prescribed tolerance.

\begin{definition}[$\varepsilon$-Ground-Truth Disturbance Mask]
	\label{def:gt}
	Let $\varepsilon > 0$ be a threshold representing the minimum meaningful disturbance amplitude. The \emph{$\varepsilon$-ground-truth disturbance mask} is the binary vector $R_{\mathrm{GT}}(x) \in \{0,1\}^N$ defined componentwise by
	\begin{equation}
		\label{eq:gt}
		R_{\mathrm{GT}}(x)[n]
		=
		\begin{cases}
			1, & \text{if } |d_x[n]| > \varepsilon, \\
			0, & \text{otherwise,}
		\end{cases}
		\qquad n = 1, \dots, N,
	\end{equation}
	where $R_{\mathrm{GT}}(x)[n] = 1$ indicates that time index $n$ lies within a disturbance region. The set of all disturbance indices is denoted $\mathcal{I}(x) = \{n : R_{\mathrm{GT}}(x)[n] = 1\}$, and its cardinality $L = |\mathcal{I}(x)|$ is the ground-truth disturbance length.
\end{definition}

The mask $R_{\mathrm{GT}}(x)$ depends on the disturbance component associated with input $x$ and on the choice of threshold $\varepsilon$. It is independent of the classifier parameters $\theta$ and the attribution operator $\tau_x$, and thus constitutes a model-agnostic evaluation reference.

\subsection{Localisation Metrics for Point Explanations}
\label{sec:eval:point}

For any real-valued saliency map $R$, two scalar metrics are adopted to measure the alignment between $R$ and $R_{\mathrm{GT}}(x)$. Throughout the quantitative evaluation, real-valued saliency maps are assumed to be nonnegative, i.e., $R \in \mathbb{R}_{\geq 0}^N$. Signed attribution maps are converted to nonnegative relevance maps prior to evaluation.

\noindent\textbf{Relevance Mass Accuracy (RMA)} The RMA of a nonnegative saliency map $R \in \mathbb{R}_{\geq 0}^N$ with respect to $R_{\mathrm{GT}}(x)$ is
	\begin{equation}
		\label{eq:rma}
		\mathrm{RMA}(R,\, x)
		=
		\frac{
			\displaystyle\sum_{n \in \mathcal{I}(x)} R[n]
		}{
			\displaystyle\sum_{n=1}^{N} R[n]
		},
	\end{equation}
	defined whenever $\sum_{n=1}^{N} R[n] \neq 0$. RMA measures the proportion of total relevance mass concentrated within the true disturbance region.

\noindent\textbf{Binarized Intersection over Union (IoU)} Let $L = |\mathcal{I}(x)|$ be the ground-truth disturbance length and assume $L>0$. The \emph{top-$L$ binarisation} of $R \in \mathbb{R}_{\geq 0}^N$ is
	\begin{equation}
		\label{eq:topL}
		\hat{R}[n]
		=
		\begin{cases}
			1, & \text{if } n \in \mathcal{T}_L(R), \\
			0, & \text{otherwise,}
		\end{cases}
	\end{equation}
	where $\mathcal{T}_L(R) \subseteq \{1, \dots, N\}$ denotes the index set of the $L$ largest values of $R$. Let $\widehat{\mathcal{I}}_R = \{n:\hat{R}[n]=1\} = \mathcal{T}_L(R)$ be the predicted disturbance index set. The IoU score is
	\begin{equation}
		\label{eq:iou}
		\mathrm{IoU}(R,\, x)
		=
		\frac{
			|\widehat{\mathcal{I}}_R \cap \mathcal{I}(x)|
		}{
			|\widehat{\mathcal{I}}_R \cup \mathcal{I}(x)|
		}
		\;\in [0, 1],
	\end{equation}
	where $|\cdot|$ denotes set cardinality, the intersection counts positions identified as disturbance by both $\widehat{\mathcal{I}}_R$ and $\mathcal{I}(x)$, and the union counts positions identified by either set. A score of $1$ indicates perfect overlap.

\begin{remark}[Consistency of the Top-$L$ Binarization]
	The top-$L$ binarisation in Eq.~\eqref{eq:topL} enforces that $|\widehat{\mathcal{I}}_R| = L = |\mathcal{I}(x)|$, so that the cardinality of the predicted mask matches the ground truth exactly. This eliminates the confounding effect of threshold selection on set size, making IoU a measure of localisation quality rather than of detection scale.
\end{remark}

\begin{remark}[Local XAI and Metric Sensitivity]
	\label{rem:metrics_type}
	The disturbance interval $\mathcal{I}(x)$ varies in duration and location across both disturbance categories and individual instances. For long-duration disturbances (e.g., harmonics), $|\mathcal{I}(x)| \approx N$ and high RMA or IoU scores are trivially achievable. For short-duration events (e.g., impulsive transients), $|\mathcal{I}(x)| \ll N$ and both metrics are sensitive discriminators. This instance-dependent structure makes local explanations, i.e., attributions specific to each input $x$, the natural object of evaluation against $R_{\mathrm{GT}}(x)$.
\end{remark}

Under the nonnegativity condition on $R$ and the requirement $\sum_{n=1}^{N}R[n]\neq0$, RMA is bounded in $[0,1]$. IoU is also bounded in $[0,1]$ for $L>0$, with larger values indicating better alignment. The two metrics are complementary: RMA is sensitive to the distribution of relevance mass and does not require binarisation, while IoU assesses spatial overlap after binarisation and is insensitive to the magnitude of non-disturbance relevance.

\subsection{Distributional Evaluation of UA-RAO Instantiations}
\label{sec:eval:dist}

The UA-RAO framework produces a family of explanation summaries $\{\widehat{\mathcal{U}}_\Phi(x)\}_\Phi$ rather than a single saliency map. The localisation metrics are extended to this distributional setting as follows.

\noindent\textbf{Point instantiations.} For UA-RAO instantiations that produce nonnegative real-valued relevance maps, specifically the mean explanation $\widehat{\mathcal{U}}^{\mathrm{mean}}$ and the $\alpha$-quantile explanation $\widehat{\mathcal{U}}^{\alpha}$, the localisation metrics are applied directly:
\begin{equation}
	\label{eq:rma_phi}
	\mathrm{RMA}_\Phi(x)
	= \mathrm{RMA}\!\left(\widehat{\mathcal{U}}_\Phi(x),\, x\right),
	\qquad
	\mathrm{IoU}_\Phi(x)
	= \mathrm{IoU}\!\left(\widehat{\mathcal{U}}_\Phi(x),\, x\right).
\end{equation}

\noindent\textbf{Set-valued instantiations.} The agreement-set UA-RAO $\widehat{\mathcal{E}}_{\eta,\delta}$ in Table~\ref{tab:uarao} produces a binary index set rather than a real-valued map. Lower values of $\eta$ give a more inclusive, union-like set, whereas higher values give a more conservative, intersection-like set. Such sets require set-based precision--recall evaluation rather than RMA or IoU as defined above. Their systematic evaluation across disturbance categories, including sensitivity to $(\eta,\delta)$, is deferred to future work.

\subsection{Qualitative Evaluation}
\label{sec:eval:qual}

Qualitative evaluation complements the quantitative metrics through visual inspection of the explanation distribution and its summaries.

\noindent\textbf{Saliency map visualisation.} Each real-valued UA-RAO output $\widehat{\mathcal{U}}_\Phi(x) \in \mathbb{R}^N$ is normalised to $[0,1]$ using min--max normalisation prior to display:
\begin{equation}
	\label{eq:minmax}
	\widetilde{\mathcal{U}}_\Phi(x)[n]
	=
	\frac{
		\widehat{\mathcal{U}}_\Phi(x)[n]
		- \min_{n'} \widehat{\mathcal{U}}_\Phi(x)[n']
	}{
		\max_{n'} \widehat{\mathcal{U}}_\Phi(x)[n']
		- \min_{n'} \widehat{\mathcal{U}}_\Phi(x)[n']
	},
\end{equation}
and rendered using the \texttt{cividis} colormap, which provides a perceptually uniform monotone mapping from $[0,1]$ to colour and remains legible for colour-vision-deficient readers. When a binary ground-truth mask is displayed, it is rendered separately using a grayscale map. The mean explanation, representative $\alpha$-quantile explanations, variance map $\widehat{\mathcal{U}}^{\mathrm{var}}(x)$, and coefficient-of-variation map $\widehat{\mathcal{U}}^{\mathrm{cv}}(x)$ are displayed for visual comparison. When evaluated quantitatively, uncertainty maps are treated as diagnostic localisation baselines rather than primary relevance explanations.

\section{Experiments}
\label{sec:experiments}

This section applies the evaluation framework from Section~\ref{sec:eval} to PQD classification experiments. The ground-truth masks, RMA, and IoU metrics defined in Section~\ref{sec:eval} are evaluated across selected BNN--attribution configurations. The experimental setup is described first, followed by aggregate comparisons, per-class localisation analysis, qualitative evaluation on both synthetic and real-world data, input-sensitivity tests under controlled distortions, and finite-sample self-coverage of explanation bands.

\subsection{Experimental Setup}
\label{sec:eval:setup}

\begin{table}[t]
	\footnotesize
	\centering
	\caption{Lightweight 1D-CNN architecture.}
	\label{tb:DCNN}
	\begin{tabular}{@{}lll@{}}
		\toprule
		Layer & Details & Activation \\
		\midrule
		\texttt{convolution\_1} & kernel: $1 \times 3$, channel: 8, stride: 1 & ReLU \\
		\texttt{convolution\_2} & kernel: $1 \times 3$, channel: 8, stride: 1 & ReLU \\
		\texttt{max\_pooling\_1} & kernel: $1 \times 3$, stride: 1 & \\
		\texttt{batch\_normalization\_1} & num\_features: 8 & \\
		\texttt{convolution\_3} & kernel: $1 \times 3$, channel: 16, stride: 1 & ReLU \\
		\texttt{convolution\_4} & kernel: $1 \times 3$, channel: 16, stride: 1 & ReLU \\
		\texttt{max\_pooling\_2} & kernel: $1 \times 630$ & \\
		\texttt{batch\_normalization\_2} & num\_features: 16 & \\
		\texttt{flatten\_1} & & \\
		\texttt{fully\_connected\_1} & num\_features: 64 & ReLU \\
		\texttt{batch\_normalization\_3} & num\_features: 64 & \\
		\texttt{fully\_connected\_2} & num\_features: 16 & Softmax \\
		\bottomrule
	\end{tabular}
\end{table}

The data are generated from the synthetic PQD setting of XPQRS \cite{khan2023xpqrs}. Each voltage record spans ten cycles of a nominal sinusoidal waveform and includes additive noise and labelled disturbance events. The dataset includes a normal (undisturbed) class. Localisation scores are computed only for the 15 disturbance classes: because a normal waveform has no disturbance component $d_x$, the ground-truth mask $R_{\mathrm{GT}}(x)$ is identically zero and RMA and IoU are undefined. The evaluation uses five independent test splits with 1500 non-normal instances per split, giving 7500 evaluated waveforms in total. Ground-truth masks are obtained from the clean disturbance reference following Section~\ref{sec:eval:gt}.

The classifier is based on the 1D-CNN architecture of Wang and Chen \cite{wang2019novel}. Because the proposed evaluation requires repeated posterior sampling and repeated local attribution computation for each waveform, we use a lightweight variant of this 1D-CNN to reduce computational cost while preserving the same convolutional design principle. Specifically, the number of convolutional channels and fully connected units is reduced, and the final temporal aggregation is implemented through a global max-pooling operation over the remaining temporal dimension. The resulting architecture is summarised in Table~\ref{tb:DCNN}. Unless noted otherwise, attribution is computed for the true class label and signed attributions are converted to absolute relevance.

\begin{table}[t]
	
	\footnotesize
	\centering
	\caption{Training hyperparameters shared by all evaluated networks (deterministic baseline, the five deep-ensemble members, and the networks underlying the Laplace and MC Dropout variants).}
	\label{tb:hyper}
	\begin{tabular}{@{}ll@{}}
		\toprule
		Hyperparameter & Value \\
		\midrule
		Optimiser & Adam ($\beta_1=0.9$, $\beta_2=0.999$) \\
		Initial learning rate & $10^{-2}$ \\
		Learning-rate schedule & step decay, $\times 0.5$ every 10 epochs \\
		Weight decay & $10^{-4}$ \\
		Batch size & 64 \\
		Training epochs & 100 \\
		Loss & categorical cross-entropy (Eq.~\eqref{eq:loss}) \\
		Training / validation split & 12{,}960 / 1{,}440 waveforms (90\%/10\%) \\
		Model selection & checkpoint with lowest validation loss \\
		Ensemble member seeds & 2026--2030 (initialisation and data shuffling) \\
		\bottomrule
	\end{tabular}
\end{table}

\noindent\textbf{Training protocol.} All networks use the protocol in Table~\ref{tb:hyper}. They are trained for 100 epochs with Adam, an initial learning rate of $10^{-2}$ halved every 10 epochs, weight decay $10^{-4}$, and batch size 64. The checkpoint with the lowest validation loss is retained. The training pool contains 12{,}960 training and 1{,}440 validation waveforms. As the lightweight architecture differs from that of Wang and Chen \cite{wang2019novel}, the protocol was selected using the validation split and is reported in full. The five ensemble members differ only in their random seeds (2026--2030), which control parameter initialisation and mini-batch shuffling.

The evaluation considers three BNN posterior approximations and three post-hoc attribution operators, covering distinct uncertainty estimation and explanation mechanisms. This design allows UA-RAO and its localisation metrics to be assessed across multiple methodological settings rather than a single configuration.

\noindent\textbf{Uncertainty estimation methods:}
\begin{itemize}
\item \emph{Deep ensembles} \cite{lakshminarayanan2017simple} train multiple independently initialised networks and use the diversity of their predictions as a non-parametric estimate of predictive uncertainty. They are well-calibrated and robust to distribution shift but incur proportionally higher training cost. Within the UA-RAO framework, each ensemble member serves as a draw from the implicit function-space distribution induced by random initialisation and stochastic optimisation, acting as a proxy for $q_\phi$ without an explicit parametric posterior.
\item \emph{MC Dropout} \cite{gal2016dropout} interprets test-time dropout as approximate Bayesian inference, generating stochastic predictions by randomly masking network units. It is computationally lightweight and imposes no architectural constraint beyond standard dropout layers.
\item \emph{Laplace approximation} \cite{ritter2018scalable} fits a Gaussian centred at the maximum \emph{a posteriori} (MAP) solution by locally approximating the log-posterior curvature. It operates on a pre-trained network without retraining and can therefore be applied post hoc to any existing classifier.
\end{itemize}

\noindent\textbf{Attribution operators.}
\begin{itemize}
\item\emph{Grad-CAM} \cite{selvaraju2017gradcam} weights the feature maps of a target convolutional layer by the gradient of the class score with respect to those maps, producing a coarse gradient-based localisation map without modifying the network.
\item\emph{LIME} \cite{ribeiro2016trust} fits a sparse linear surrogate model in the neighbourhood of an input instance via perturbation-based sampling, yielding a local, model-agnostic explanation that requires no access to network gradients.
\item\emph{Occlusion} \cite{zeiler2014visualizing} slides a masking window across the input and records the change in class probability at each position, providing a perturbation-based relevance score that requires neither gradient access nor a differentiable model.
These three operators represent gradient-based, surrogate-model, and perturbation-based explanation mechanisms. Their structural diversity ensures that the reported UA-RAO results are not artefacts of a single attribution family. 
\end{itemize}

The detailed experimental settings are as follows. Occlusion uses a window length of 60 and stride 1, LIME uses 128 perturbation samples with feature width 16, and Grad-CAM is taken from the fourth convolutional layer. Signed saliency maps (Grad-CAM) are converted to nonnegative relevance via the absolute value before applying RMA and IoU, consistent with the nonnegativity assumption in Section~\ref{sec:eval:point}. Bayesian variants are evaluated with the same attribution operator. The deep ensemble contains five independently trained members. The Laplace approximation follows the scalable neural-network Laplace procedure of Ritter et al.\ \cite{ritter2018scalable} with a diagonal Fisher approximation. Its posterior precision is $\Lambda = s\,F + \delta I$, where $F$ is the accumulated diagonal Fisher and $I$ is the identity matrix. The damping term $\delta = 10^2$ corresponds to a zero-mean isotropic Gaussian prior $p(\theta) = \mathcal{N}(0,\,\delta^{-1}I)$ with standard deviation $0.1$, and it also stabilises near-zero curvature entries. Following the post-hoc procedure in \cite{ritter2018scalable}, the factor $s = 1.75\times10^{10}$ is selected so that sampled predictive confidence matches empirical validation confidence. The large value of $s$ indicates that the raw diagonal Fisher underestimates curvature for this lightweight architecture; without rescaling, the posterior variance is too large. Both $\delta$ and $s$ are therefore reported. MC Dropout uses dropout probability $0.2$. Laplace approximation and MC Dropout each use 20 posterior samples. Table~\ref{tb:hyper} lists the remaining hyperparameters.

\subsection{Numerical Results and Discussion}
\label{sec:eval:results}

The per-class analysis in Section~\ref{sec:eval:perclass} shows that eight of the 15 disturbance categories (Flicker, Flicker+Harmonics, Flicker+Sag, Flicker+Swell, Harmonics, Interruption+Harmonics, Sag+Harmonics, and Swell+Harmonics) reach IoU $\approx 0.997$ for almost every method. Their disturbance regions cover most of the waveform, so a broad relevance map can obtain near-unity overlap. To keep these ceiling-effect classes from dominating the comparison, the aggregate tables also report macro RMA and IoU for the seven discriminative classes: Impulsive Transient, Interruption, Notch, Oscillatory Transient, Sag, Spike, and Swell. These disc-7 scores are the unweighted means of the class-wise scores, averaged over the five test splits.

\begin{table}[H]
	\begin{adjustwidth}{-2.0cm}{-2.0cm}
	\centering
	
	\caption{Comparison of deterministic and Bayesian occlusion explanations using mean UA-RAO. Results are mean $\pm$ standard deviation across five test splits. The disc-7 columns report unweighted macro scores after excluding the eight ceiling-effect classes defined in Section~\ref{sec:eval:results}. \textbf{Bold} indicates the best value in each metric column.}
	\label{tab:bnn_results}
	\small
	\resizebox{\linewidth}{!}{%
	\begin{tabular}{@{}lccccc@{}}
		\toprule
		Method & Accuracy & RMA & IoU & RMA (disc-7) & IoU (disc-7) \\
		\midrule
		Deterministic baseline & $0.9774 \pm 0.0014$ & $0.6132 \pm 0.0034$ & $0.6514 \pm 0.0039$ & $0.1730 \pm 0.0073$ & $0.2565 \pm 0.0082$ \\
		Deep ensemble & $\mathbf{0.9813 \pm 0.0023}$ & $0.6012 \pm 0.0018$ & $\mathbf{0.6711 \pm 0.0043}$ & $0.1471 \pm 0.0038$ & $\mathbf{0.2985 \pm 0.0092}$ \\
		Laplace & $0.9771 \pm 0.0021$ & $0.6096 \pm 0.0019$ & $0.6518 \pm 0.0012$ & $0.1652 \pm 0.0041$ & $0.2573 \pm 0.0026$ \\
		MC Dropout & $0.9769 \pm 0.0022$ & $\mathbf{0.6183 \pm 0.0020}$ & $0.6493 \pm 0.0026$ & $\mathbf{0.1839 \pm 0.0042}$ & $0.2517 \pm 0.0057$ \\
		\bottomrule
	\end{tabular}
	}%
	\end{adjustwidth}
\end{table}

\begin{table}[H]
	\centering
	
	\caption{Comparison of XAI operators using deep ensemble and mean UA-RAO. Results are mean $\pm$ standard deviation across five test splits. The disc-7 columns report macro scores over the seven discriminative classes. \textbf{Bold} indicates the best value in each metric column.}
	\label{tab:xai_results}
	\small
	\resizebox{\linewidth}{!}{%
	\begin{tabular}{@{}lcccc@{}}
		\toprule
		XAI operator & RMA & IoU & RMA (disc-7) & IoU (disc-7) \\
		\midrule
		Grad-CAM & $0.7190 \pm 0.0030$ & $0.6232 \pm 0.0024$ & $0.3979 \pm 0.0064$ & $0.1961 \pm 0.0052$ \\
		LIME & $\mathbf{0.7476 \pm 0.0262}$ & $0.6365 \pm 0.0019$ & $\mathbf{0.4094 \pm 0.0320}$ & $0.2247 \pm 0.0040$ \\
		Occlusion & $0.6012 \pm 0.0018$ & $\mathbf{0.6716 \pm 0.0044}$ & $0.1472 \pm 0.0038$ & $\mathbf{0.2995 \pm 0.0093}$ \\
		\bottomrule
	\end{tabular}
	}%
	\par\medskip
	\begin{minipage}{\linewidth}
		\footnotesize\textit{Note:} The occlusion row is an independent re-evaluation of the deep-ensemble--occlusion--mean configuration in Tables~\ref{tab:bnn_results} and~\ref{tab:uarao_results}. The small numerical differences arise from variation between the separate frozen runs.
	\end{minipage}
\end{table}

\begin{table}[H]
	\centering
	
	\caption{Comparison of UA-RAO summaries using deep ensemble and occlusion. Results are mean $\pm$ standard deviation across five test splits. The disc-7 columns report macro scores over the seven discriminative classes. \textbf{Bold} indicates the best value in each metric column.}
	\label{tab:uarao_results}
	\small
	\resizebox{\linewidth}{!}{%
	\begin{tabular}{@{}lcccc@{}}
		\toprule
		UA-RAO summary & RMA & IoU & RMA (disc-7) & IoU (disc-7) \\
		\midrule
		Mean & $0.6012 \pm 0.0018$ & $\mathbf{0.6711 \pm 0.0043}$ & $0.1471 \pm 0.0038$ & $\mathbf{0.2985 \pm 0.0092}$ \\
		Variance & $0.5831 \pm 0.0018$ & $0.5795 \pm 0.0027$ & $0.1100 \pm 0.0039$ & $0.1024 \pm 0.0058$ \\
		Coeff. variation & $0.6084 \pm 0.0030$ & $0.5981 \pm 0.0049$ & $0.1632 \pm 0.0064$ & $0.1422 \pm 0.0104$ \\
		$5\%$ quantile & $\mathbf{0.6228 \pm 0.0042}$ & $0.6592 \pm 0.0029$ & $\mathbf{0.1930 \pm 0.0091}$ & $0.2727 \pm 0.0061$ \\
		$25\%$ quantile & $0.6178 \pm 0.0039$ & $0.6632 \pm 0.0016$ & $0.1824 \pm 0.0084$ & $0.2816 \pm 0.0034$ \\
		$50\%$ quantile & $0.6093 \pm 0.0033$ & $0.6553 \pm 0.0025$ & $0.1644 \pm 0.0071$ & $0.2646 \pm 0.0053$ \\
		$75\%$ quantile & $0.6044 \pm 0.0021$ & $0.6365 \pm 0.0043$ & $0.1538 \pm 0.0046$ & $0.2242 \pm 0.0092$ \\
		$95\%$ quantile & $0.5957 \pm 0.0021$ & $0.6487 \pm 0.0063$ & $0.1354 \pm 0.0046$ & $0.2505 \pm 0.0136$ \\
		\bottomrule
	\end{tabular}
	}%
\end{table}

Table~\ref{tab:bnn_results} compares deterministic occlusion with three model-distribution methods using the mean UA-RAO. Deep-ensemble mean occlusion increases disc-7 macro IoU from $0.2565$ to $0.2985$. The mean increase across the five matched test splits is $0.042$ (95\% confidence interval $0.035$--$0.049$), and all five splits show an increase. The result varies by class: IoU improves for impulsive transient, notch, sag, spike, and swell, but not for interruption or oscillatory transient. MC Dropout gives the highest RMA but a lower IoU, showing that relevance-mass concentration and spatial overlap measure different properties. This is an empirical comparison of the evaluated model distributions. Proposition~\ref{prop:bias} concerns explanation-distribution error, whereas IoU measures overlap with the physical disturbance mask; the proposition therefore does not determine the empirical ranking.

Table~\ref{tab:xai_results} compares the three attribution operators with the deep ensemble and mean UA-RAO fixed. Predictive accuracy is the same for all rows and is omitted. Occlusion gives the highest disc-7 IoU ($0.2995$), while LIME gives the highest disc-7 RMA ($0.4094$). This difference again reflects the distinction between relevance-mass concentration and top-$L$ spatial overlap. Grad-CAM performs well for interruption and notch, but its coarse temporal resolution is less suitable for several short events. LIME can concentrate relevance without matching the full event interval. The ranking applies to the architecture, dataset, target class, and implementation used in this study.

\begin{table}[!p]
	\begin{adjustwidth}{-1.5cm}{-1.5cm}
		\centering
		\captionsetup{font=small,skip=2pt}
		\setlength{\tabcolsep}{2pt}
		\renewcommand{\arraystretch}{0.72}
		\caption{Per-class IoU for deterministic and Bayesian occlusion explanations. \textbf{Bold} marks the unique row best.}
		\label{tab:perclass_bnn}
		\scriptsize
		\begin{tabular}{@{}lcccc@{}}
			\toprule
			Disturbance & Det. & Deep ens. & Laplace & MC Drop. \\
			\midrule
			Flicker & $0.997 \pm 0.000$ & $0.997 \pm 0.000$ & $0.997 \pm 0.000$ & $0.997 \pm 0.000$ \\
			Flicker+Harm. & $0.997 \pm 0.000$ & $\mathbf{0.998 \pm 0.000}$ & $0.997 \pm 0.000$ & $0.997 \pm 0.000$ \\
			Flicker+Sag & $0.997 \pm 0.000$ & $0.997 \pm 0.000$ & $0.997 \pm 0.000$ & $\mathbf{0.998 \pm 0.000}$ \\
			Flicker+Swell & $0.997 \pm 0.000$ & $0.997 \pm 0.000$ & $0.997 \pm 0.000$ & $0.997 \pm 0.000$ \\
			Harmonics & $0.997 \pm 0.000$ & $0.997 \pm 0.000$ & $0.997 \pm 0.000$ & $0.997 \pm 0.000$ \\
			Imp. trans. & $0.005 \pm 0.005$ & $\mathbf{0.104 \pm 0.018}$ & $0.006 \pm 0.006$ & $0.004 \pm 0.003$ \\
			Interruption & $0.112 \pm 0.010$ & $0.105 \pm 0.018$ & $0.141 \pm 0.004$ & $\mathbf{0.163 \pm 0.015}$ \\
			Interruption+Harm. & $0.997 \pm 0.000$ & $0.997 \pm 0.000$ & $0.997 \pm 0.000$ & $\mathbf{0.998 \pm 0.000}$ \\
			Notch & $0.010 \pm 0.004$ & $\mathbf{0.020 \pm 0.002}$ & $0.011 \pm 0.004$ & $0.007 \pm 0.002$ \\
			Osc. trans. & $0.654 \pm 0.011$ & $0.601 \pm 0.018$ & $\mathbf{0.665 \pm 0.016}$ & $0.635 \pm 0.007$ \\
			Sag & $0.291 \pm 0.037$ & $\mathbf{0.364 \pm 0.030}$ & $0.281 \pm 0.034$ & $0.259 \pm 0.026$ \\
			Sag+Harm. & $0.997 \pm 0.000$ & $0.998 \pm 0.000$ & $0.998 \pm 0.000$ & $0.998 \pm 0.000$ \\
			Spike & $0.106 \pm 0.015$ & $\mathbf{0.123 \pm 0.015}$ & $0.078 \pm 0.020$ & $0.034 \pm 0.004$ \\
			Swell & $0.618 \pm 0.014$ & $\mathbf{0.773 \pm 0.010}$ & $0.619 \pm 0.012$ & $0.659 \pm 0.011$ \\
			Swell+Harm. & $0.997 \pm 0.000$ & $0.997 \pm 0.000$ & $0.997 \pm 0.000$ & $0.997 \pm 0.000$ \\
			\bottomrule
		\end{tabular}
	\end{adjustwidth}
	\vspace{-1.0em}
\end{table}

\begin{table}[!p]
	\begin{adjustwidth}{-1.0cm}{-1.0cm}
		\centering
		\captionsetup{font=small,skip=2pt}
		\setlength{\tabcolsep}{2pt}
		\renewcommand{\arraystretch}{0.72}
		\caption{Per-class IoU for XAI operators with deep ensemble mean UA-RAO. \textbf{Bold} marks the unique row best.}
		\label{tab:perclass_xai}
		\scriptsize
		\begin{tabular}{@{}lccc@{}}
			\toprule
			Disturbance & Grad-CAM & LIME & Occlusion \\
			\midrule
			Flicker & $0.997 \pm 0.000$ & $0.997 \pm 0.000$ & $0.997 \pm 0.000$ \\
			Flicker+Harm. & $0.997 \pm 0.000$ & $0.997 \pm 0.000$ & $\mathbf{0.998 \pm 0.000}$ \\
			Flicker+Sag & $0.997 \pm 0.000$ & $0.997 \pm 0.000$ & $0.997 \pm 0.000$ \\
			Flicker+Swell & $0.997 \pm 0.000$ & $0.997 \pm 0.000$ & $0.997 \pm 0.000$ \\
			Harmonics & $0.997 \pm 0.000$ & $0.997 \pm 0.000$ & $0.997 \pm 0.000$ \\
			Imp. trans. & $0.101 \pm 0.022$ & $0.003 \pm 0.004$ & $\mathbf{0.108 \pm 0.024}$ \\
			Interruption & $\mathbf{0.246 \pm 0.015}$ & $0.129 \pm 0.014$ & $0.105 \pm 0.019$ \\
			Interruption+Harm. & $0.997 \pm 0.000$ & $0.997 \pm 0.000$ & $0.997 \pm 0.000$ \\
			Notch & $\mathbf{0.060 \pm 0.002}$ & $0.019 \pm 0.002$ & $0.020 \pm 0.002$ \\
			Osc. trans. & $0.329 \pm 0.013$ & $0.591 \pm 0.013$ & $\mathbf{0.601 \pm 0.018}$ \\
			Sag & $0.273 \pm 0.017$ & $0.318 \pm 0.013$ & $\mathbf{0.364 \pm 0.030}$ \\
			Sag+Harm. & $0.997 \pm 0.000$ & $0.997 \pm 0.000$ & $\mathbf{0.998 \pm 0.000}$ \\
			Spike & $0.055 \pm 0.001$ & $0.026 \pm 0.002$ & $\mathbf{0.126 \pm 0.017}$ \\
			Swell & $0.308 \pm 0.012$ & $0.486 \pm 0.028$ & $\mathbf{0.773 \pm 0.010}$ \\
			Swell+Harm. & $0.997 \pm 0.000$ & $0.997 \pm 0.000$ & $0.997 \pm 0.000$ \\
			\bottomrule
		\end{tabular}
	\end{adjustwidth}
	\vspace{-1.0em}
\end{table}

\begin{table}[!p]
	\begin{adjustwidth}{-2.7cm}{-2.7cm}
		\centering
		\captionsetup{font=small,skip=2pt}
		\setlength{\tabcolsep}{1.0pt}
		\renewcommand{\arraystretch}{0.72}
		\caption{Per-class IoU for UA-RAO summaries with deep ensemble occlusion. \textbf{Bold} marks the unique row best.}
		\label{tab:perclass_uarao}
		\scriptsize
		\begin{tabular}{@{}lcccccccc@{}}
			\toprule
			Disturbance & Mean & Var. & Coeff. var. & $5\%$ & $25\%$ & $50\%$ & $75\%$ & $95\%$ \\
			\midrule
			Flicker & $0.997 \pm 0.000$ & $0.997 \pm 0.000$ & $0.997 \pm 0.000$ & $0.997 \pm 0.000$ & $0.997 \pm 0.000$ & $0.997 \pm 0.000$ & $0.997 \pm 0.000$ & $0.997 \pm 0.000$ \\
			Flicker+Harm. & $0.998 \pm 0.000$ & $0.997 \pm 0.000$ & $0.997 \pm 0.000$ & $0.998 \pm 0.000$ & $0.997 \pm 0.000$ & $0.997 \pm 0.000$ & $0.997 \pm 0.000$ & $0.997 \pm 0.000$ \\
			Flicker+Sag & $0.997 \pm 0.000$ & $0.997 \pm 0.000$ & $0.997 \pm 0.000$ & $0.997 \pm 0.000$ & $0.997 \pm 0.000$ & $0.997 \pm 0.000$ & $0.997 \pm 0.000$ & $0.997 \pm 0.000$ \\
			Flicker+Swell & $0.997 \pm 0.000$ & $0.997 \pm 0.000$ & $0.997 \pm 0.000$ & $0.997 \pm 0.000$ & $0.997 \pm 0.000$ & $0.997 \pm 0.000$ & $0.997 \pm 0.000$ & $0.997 \pm 0.000$ \\
			Harmonics & $0.997 \pm 0.000$ & $0.997 \pm 0.000$ & $0.997 \pm 0.000$ & $\mathbf{0.998 \pm 0.000}$ & $0.997 \pm 0.000$ & $0.997 \pm 0.000$ & $0.997 \pm 0.000$ & $0.997 \pm 0.000$ \\
			Imp. trans. & $\mathbf{0.104 \pm 0.018}$ & $0.027 \pm 0.009$ & $0.036 \pm 0.009$ & $0.049 \pm 0.012$ & $0.063 \pm 0.020$ & $0.061 \pm 0.037$ & $0.051 \pm 0.016$ & $0.069 \pm 0.018$ \\
			Interruption & $0.105 \pm 0.018$ & $0.128 \pm 0.028$ & $\mathbf{0.332 \pm 0.059}$ & $0.104 \pm 0.016$ & $0.108 \pm 0.020$ & $0.108 \pm 0.015$ & $0.112 \pm 0.017$ & $0.123 \pm 0.026$ \\
			Interruption+Harm. & $0.997 \pm 0.000$ & $0.997 \pm 0.000$ & $0.997 \pm 0.000$ & $0.997 \pm 0.000$ & $0.997 \pm 0.000$ & $0.997 \pm 0.000$ & $0.997 \pm 0.000$ & $0.997 \pm 0.000$ \\
			Notch & $0.020 \pm 0.002$ & $0.034 \pm 0.004$ & $0.034 \pm 0.004$ & $0.009 \pm 0.003$ & $0.024 \pm 0.004$ & $0.035 \pm 0.006$ & $0.025 \pm 0.005$ & $\mathbf{0.039 \pm 0.009}$ \\
			Osc. trans. & $0.601 \pm 0.018$ & $0.098 \pm 0.013$ & $0.053 \pm 0.009$ & $\mathbf{0.638 \pm 0.014}$ & $0.613 \pm 0.018$ & $0.580 \pm 0.018$ & $0.547 \pm 0.027$ & $0.491 \pm 0.028$ \\
			Sag & $0.364 \pm 0.030$ & $0.287 \pm 0.009$ & $\mathbf{0.398 \pm 0.050}$ & $0.259 \pm 0.028$ & $0.263 \pm 0.029$ & $0.271 \pm 0.030$ & $0.290 \pm 0.025$ & $0.370 \pm 0.028$ \\
			Sag+Harm. & $0.998 \pm 0.000$ & $0.997 \pm 0.000$ & $0.997 \pm 0.000$ & $0.997 \pm 0.000$ & $0.997 \pm 0.000$ & $0.997 \pm 0.000$ & $0.998 \pm 0.000$ & $0.997 \pm 0.000$ \\
			Spike & $0.123 \pm 0.015$ & $0.028 \pm 0.008$ & $0.027 \pm 0.008$ & $0.084 \pm 0.015$ & $0.142 \pm 0.009$ & $0.192 \pm 0.025$ & $0.191 \pm 0.018$ & $\mathbf{0.213 \pm 0.021}$ \\
			Swell & $\mathbf{0.773 \pm 0.010}$ & $0.115 \pm 0.011$ & $0.115 \pm 0.011$ & $0.765 \pm 0.012$ & $0.759 \pm 0.014$ & $0.604 \pm 0.035$ & $0.353 \pm 0.012$ & $0.449 \pm 0.022$ \\
			Swell+Harm. & $0.997 \pm 0.000$ & $0.997 \pm 0.000$ & $0.997 \pm 0.000$ & $0.997 \pm 0.000$ & $0.997 \pm 0.000$ & $0.997 \pm 0.000$ & $0.997 \pm 0.000$ & $0.997 \pm 0.000$ \\
			\bottomrule
		\end{tabular}
	\end{adjustwidth}
\end{table}

Table~\ref{tab:uarao_results} compares UA-RAO summaries for deep-ensemble occlusion. The mean gives the highest disc-7 IoU ($0.2985$), while the $5\%$ quantile gives the highest RMA. A low quantile acts as a consensus filter. It retains relevance in the disturbance core, where ensemble members agree on the deviation from the nominal waveform, and reduces relevance near onset and recovery positions where they disagree. The remaining relevance is therefore more concentrated within the physical mask, which can increase RMA. However, top-$L$ selection always chooses the same number of positions. If true boundary positions are suppressed, lower-ranked positions outside the mask may be selected instead, which reduces IoU. Variance and coefficient of variation give weaker direct localisation results and are used as summaries of explanation dispersion. The mean averages member-wise explanations; for a nonlinear explainer, it does not generally equal the explanation of the ensemble-average prediction.
\subsection{Per-Class Localisation Analysis}
\label{sec:eval:perclass}
To examine whether the macro-level trends are consistent across disturbance types, Tables~\ref{tab:perclass_bnn}--\ref{tab:perclass_uarao} report per-class IoU scores. The values are computed as mean $\pm$ standard deviation of the class-wise IoU over the five test splits. IoU is used for this per-class comparison because it is applicable to deterministic explanations and real-valued UA-RAO summaries, and directly measures spatial overlap with the ground-truth disturbance interval.

The per-class results partition the disturbance types into distinct performance regimes. Long-duration and composite disturbances, including flicker, harmonics, and harmonic mixtures, achieve IoU scores near unity for nearly all point-valued explanations. These classes occupy a large portion of the waveform or contain persistent harmonic components, so the top-$L$ mask is relatively easy to align with the ground-truth disturbance interval. For these classes, $|\mathcal{I}(x)|\approx N$, so methods assigning relevance broadly across the waveform can obtain near-unity IoU by construction. These ceiling-effect classes therefore provide limited discrimination between methods. The disc-7 macro metrics in Tables~\ref{tab:bnn_results}--\ref{tab:uarao_results} therefore exclude the eight ceiling-effect classes and focus on the seven challenging events. As shown in Section~\ref{sec:eval:results}, the improvement from the deep ensemble is larger for disc-7 than for the all-class average.

Among the seven discriminative classes, the deep ensemble improves deterministic IoU for impulsive transient, notch, sag, spike, and swell, but gives lower IoU for interruption and oscillatory transient. IoU remains low for impulsive transient and notch, for which the 60-sample occlusion window is wide relative to the event length. Laplace and MC Dropout follow different class-specific patterns. The higher macro average therefore does not indicate better performance for every disturbance type.

The XAI comparison shows that occlusion is the most reliable operator for temporally localised events. It outperforms Grad-CAM and LIME on impulsive transient, oscillatory transient, sag, spike, and swell. Grad-CAM is strongest for interruption and notch, but its performance is weaker for sag, swell, and oscillatory transient. This supports the use of occlusion as the default attribution operator for PQD localisation, because it directly measures the effect of removing contiguous waveform segments.

The UA-RAO comparison shows how posterior summaries change localisation behaviour. The $5\%$ quantile improves oscillatory transient relative to the mean explanation, indicating that conservative posterior agreement can sharpen localisation for some transient patterns. Higher quantiles improve notch and spike, where broader posterior support can recover disturbance regions that are only intermittently emphasised.

Coefficient-of-variation maps can align visually with interruption or sag boundaries, but they describe disagreement rather than direct relevance. The finite-sample theorem and the localisation metrics address different questions: the theorem concerns bounded Lipschitz summaries of an explanation distribution, while RMA and IoU measure overlap with a physical disturbance region.

\subsection{Qualitative Evaluation on Synthetic Data}
\label{sec:qual_synthetic}
\begin{figure}[!htbp]
	\begin{adjustwidth}{-2cm}{-2cm}
		\centering
		\begin{subfigure}{0.4\linewidth}
			\centering
			\includegraphics[width=\linewidth]{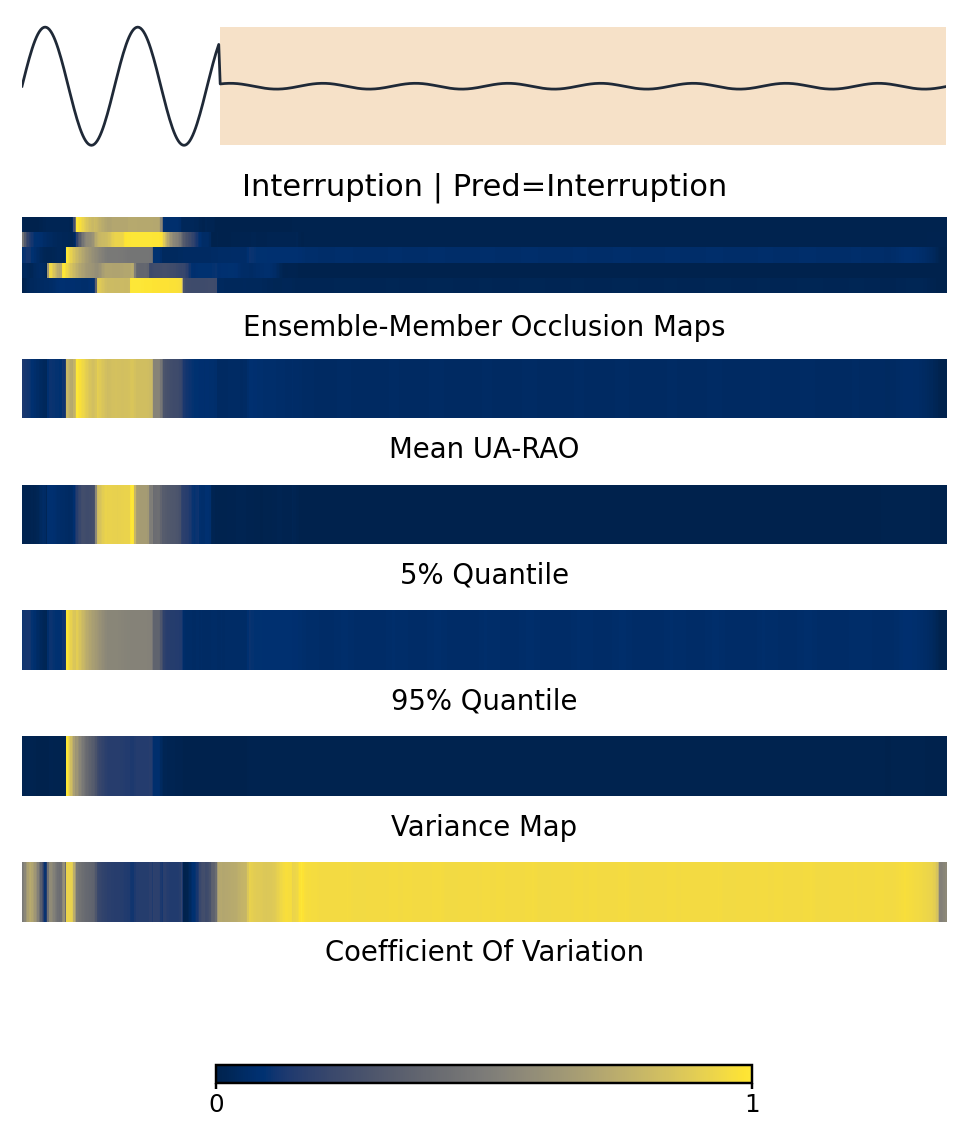}
			\caption{Interruption}
		\end{subfigure}
		\hfill
		\begin{subfigure}{0.4\linewidth}
			\centering
			\includegraphics[width=\linewidth]{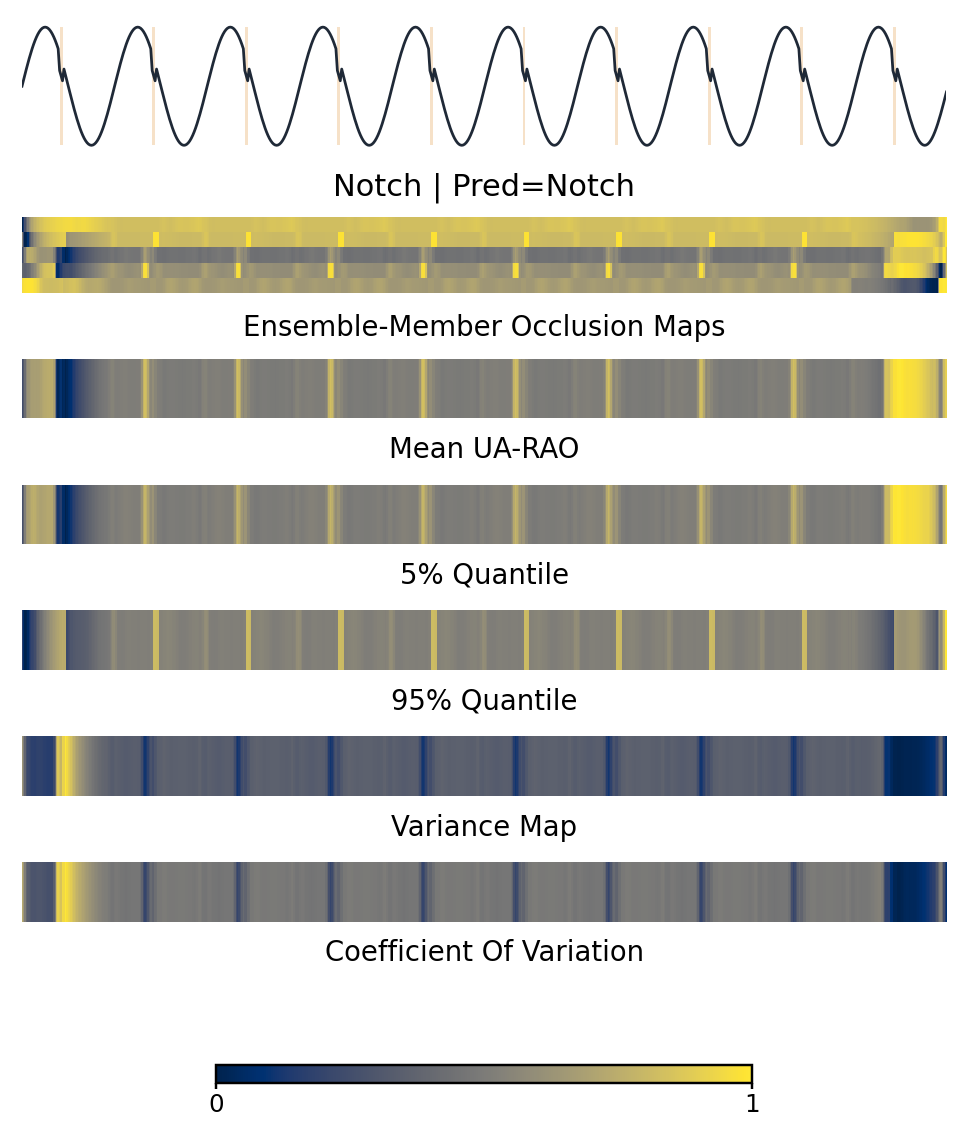}
			\caption{Notch}
		\end{subfigure}
		
		\vspace{0.3em}
		\begin{subfigure}{0.4\linewidth}
			\centering
			\includegraphics[width=\linewidth]{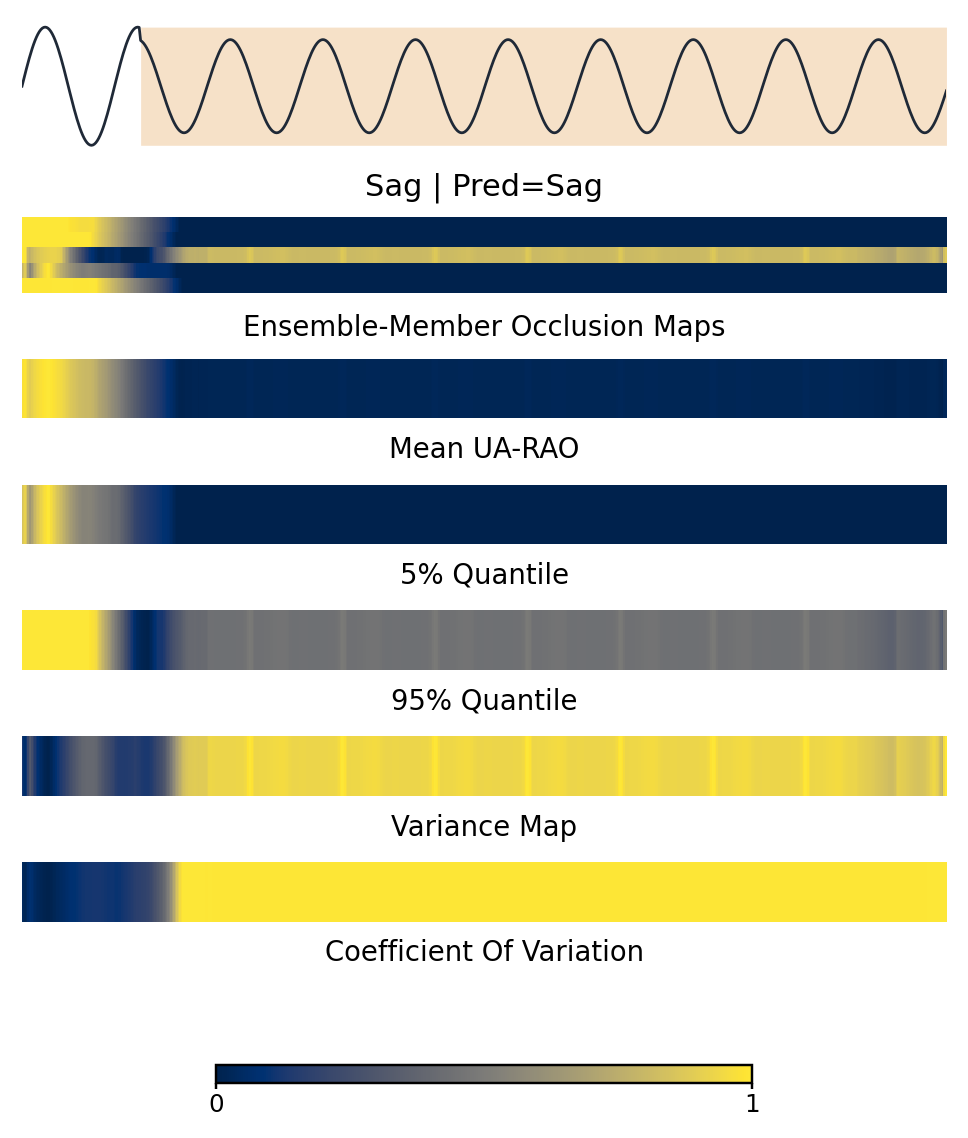}
			\caption{Sag}
		\end{subfigure}
		\hfill
		\begin{subfigure}{0.4\linewidth}
			\centering
			\includegraphics[width=\linewidth]{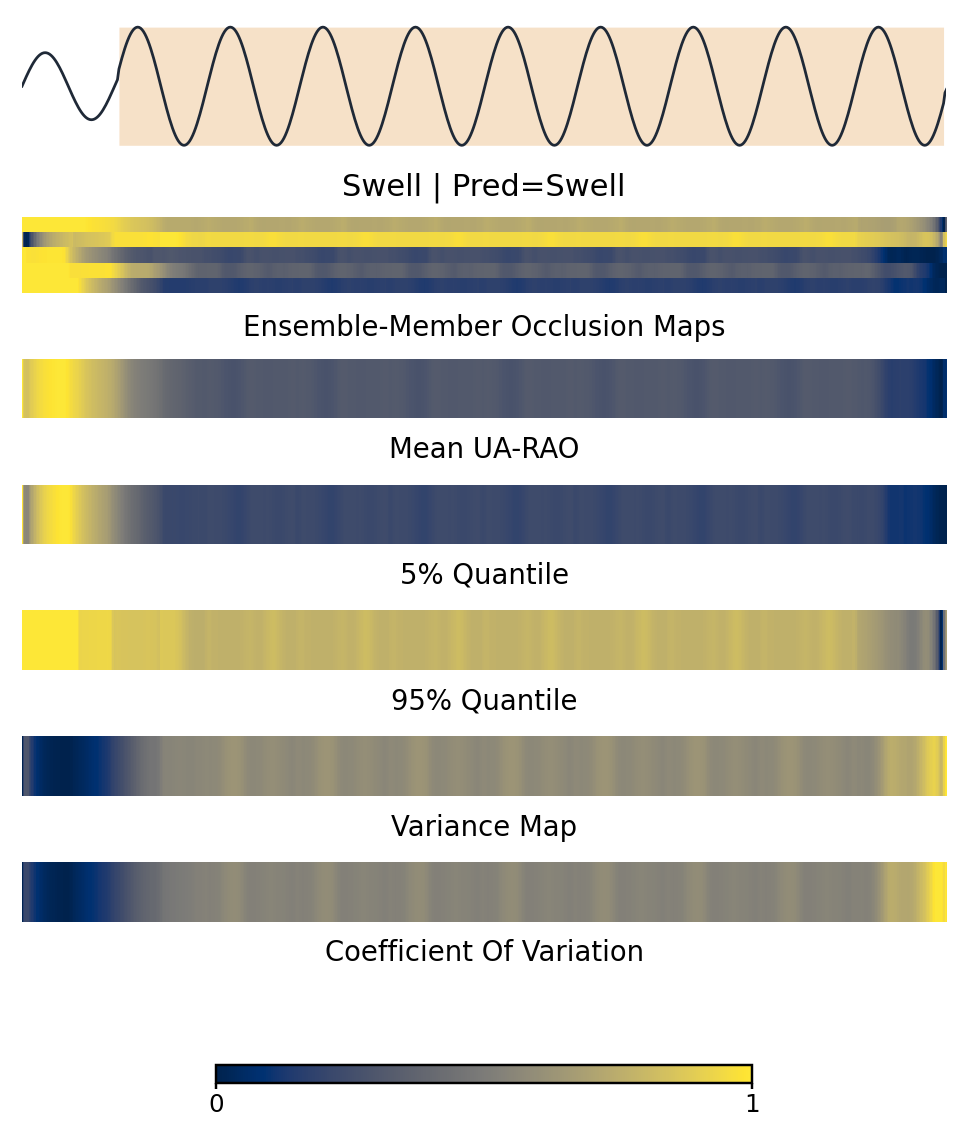}
			\caption{Swell}
		\end{subfigure}
		\caption{Qualitative synthetic examples illustrating the selected per-class UA-RAO behaviour reported in Table~\ref{tab:perclass_uarao}. Each panel shows the input waveform, ensemble-member occlusion maps, and real-valued UA-RAO summaries, including the mean explanation, quantile explanations, variance map, and coefficient of variation.}
		\label{fig:qual_synthetic}
	\end{adjustwidth}
\end{figure}
Figure~\ref{fig:qual_synthetic} provides instance-level visual evidence for the per-class UA-RAO trends summarised in Table~\ref{tab:perclass_uarao}. The four disturbance types are selected because they represent the most informative cases, excluding ceiling-effect classes, in the per-class evaluation. Together they span four qualitatively distinct uncertainty regimes. Interruption and sag illustrate cases where attribution dispersion, especially the coefficient of variation, is spatially meaningful and can align with disturbance transition regions. Notch represents a short and intermittent disturbance for which broader posterior support, particularly high-quantile summaries, can recover relevance missed by the mean explanation. Swell represents a distinctive amplitude disturbance for which the mean and low-quantile explanations remain concentrated on the true disturbed interval. These examples span ambiguous, localised, intermittent, and distinctive disturbance regimes and collectively illustrate the main per-class trends.

\begin{figure}[!htb]
	\begin{adjustwidth}{-2cm}{-2cm}
		\centering
		\begin{subfigure}{0.46\linewidth}
			\centering
			\includegraphics[width=\linewidth]{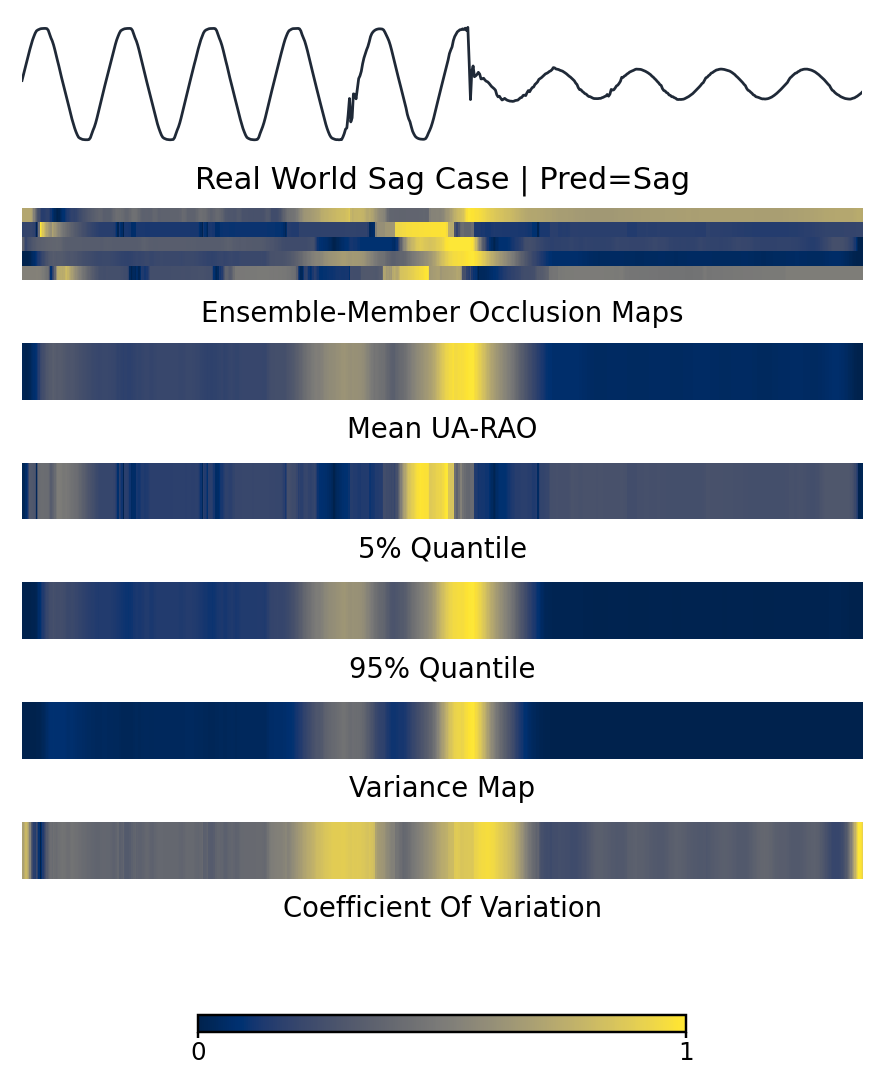}
			\caption{Real sag window 1}
		\end{subfigure}
		\hfill
		\begin{subfigure}{0.46\linewidth}
			\centering
			\includegraphics[width=\linewidth]{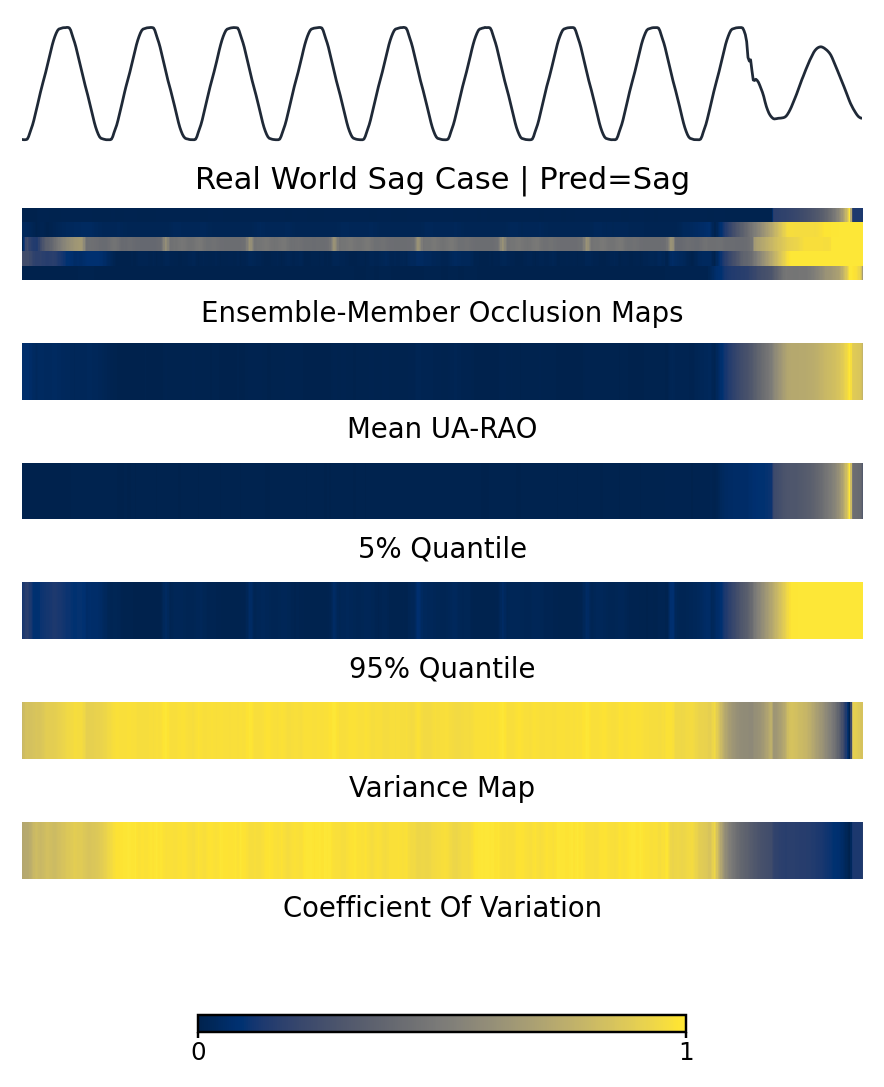}
			\caption{Real sag window 2}
		\end{subfigure}
		\caption{Qualitative real-world sag examples. Each panel shows the measured waveform, the five ensemble-member occlusion maps, and five UA-RAO summaries: the mean explanation, the $5\%$ and $95\%$ quantile explanations, the variance map, and the coefficient of variation. The two panels show how these UA-RAO summaries behave when the synthetic-trained classifier is applied to measured sag events.}
		\label{fig:qual_real}
	\end{adjustwidth}
\end{figure}

The interruption example in Figure~\ref{fig:qual_synthetic}a shows that posterior disagreement is concentrated near the disturbance transition region. This explains why the coefficient-of-variation summary performs well for interruption in Table~\ref{tab:perclass_uarao}: although coefficient of variation is not a direct relevance explanation, its high values indicate locations where ensemble members disagree about the temporal evidence for the interruption class. The notch example in Figure~\ref{fig:qual_synthetic}b illustrates the behaviour of higher quantiles for short and intermittent disturbances. The mean explanation is diffuse, whereas the $95\%$ quantile preserves several narrow high-relevance segments, matching the numerical observation that the high-quantile summary is strongest for notch. The sag example in Figure~\ref{fig:qual_synthetic}c shows a different uncertainty pattern: variance and coefficient-of-variation maps align with the sag interval and its boundary regions, supporting the interpretation that amplitude-transition regions produce elevated inter-member attribution variability. Finally, the swell example in Figure~\ref{fig:qual_synthetic}d shows that the mean and $5\%$ quantile remain concentrated over the visually disturbed interval, while more inclusive summaries become broader. This visually supports the finding that mean and low-quantile summaries are more suitable for distinctive swell events.
\subsection{Qualitative Evaluation on Real-World Data}
\label{sec:qual_real}
Figure~\ref{fig:qual_real} extends the analysis to two of these sag windows using the same deep-ensemble occlusion pipeline, examining whether the qualitative patterns observed in Table~\ref{tab:perclass_uarao} persist under real measurement conditions. The real-world data are drawn from the University of Cadiz power network recordings \cite{h2k88d-17}, which span five years of continuous measurements at a nominal frequency of 50~Hz sampled at 20~kHz, capturing diverse sag occurrences under operational conditions. For compatibility with the trained 1D-CNN, the recordings are resampled from 20~kHz to 3.2~kHz by polyphase filtering with a rational factor of $4/25$. A zero-phase finite-impulse-response (FIR) low-pass filter is applied before rate reduction to avoid aliasing. The voltage is then divided by a fixed nominal value of 300~V to match the per-unit amplitude range of the synthetic training data. Finally, the recordings are divided into 640-sample windows, corresponding to ten cycles at 50~Hz and matching the training input length. Ten-cycle windows containing clear sag events are selected manually and passed to the pre-trained model without fine-tuning. The two windows in Figure~\ref{fig:qual_real} come from different sag recordings. They illustrate a lower-confidence case with non-negligible probability for a competing class and a higher-confidence case with more concentrated relevance.

Both windows are classified as sag, but with different predictive confidence levels. The first exhibits lower sag probability with a competing oscillatory-transient probability, and its attribution maps emphasise a local deformation near the visible waveform change. The second shows higher sag confidence and more concentrated mean and quantile relevance near the visible sag-related segment. The examples show both individual ensemble explanations and UA-RAO summaries for measured inputs. The variance and coefficient-of-variation maps describe disagreement among ensemble members.

\subsection{Sensitivity of the Explanation Distribution to Input Distortions}
\label{sec:sensitivity}

We test local input sensitivity on $7{,}214$ waveforms from all 15 disturbance classes (all-15), including $3{,}366$ disc-7 cases. The deep-ensemble occlusion pipeline is evaluated under six controlled distortions: additive white Gaussian noise at 30 and 20~dB, amplitude scaling by factors $1.05$ and $0.95$, and circular shifts of $\pm4$ samples. Clean and distorted explanations are compared using four measures: correlation between their mean maps, IoU between their top-$L$ masks, mean coordinatewise $W_1$ between their five-member relevance distributions, and change in physical-mask IoU. The coordinatewise $W_1$ measure estimates $D_{\mathrm{coord}}$ in Corollary~\ref{cor:input}. For circular shifts, the explanations are aligned before comparison. This alignment is a coordinate permutation and therefore an $\ell_1$-isometry.

\begin{table}[H]
	
	\begin{adjustwidth}{-2.0cm}{-2.0cm}
	\footnotesize
	\centering
	\caption{Sensitivity of deep-ensemble occlusion explanations to input distortions (mean $\pm$ standard deviation across five test splits). Corr.\ is the Pearson correlation between clean and distorted mean maps; stability is the IoU between their top-$L$ masks; $W_1$ shift is the mean coordinatewise Wasserstein distance between the five-member relevance distributions; and $\Delta$IoU is the change in ground-truth localisation IoU.}
	\label{tab:sensitivity}
	\begin{tabular}{@{}llcccc@{}}
		\toprule
		Distortion & Scope & Corr. & Top-$L$ stability & $W_1$ shift & $\Delta$IoU \\
		\midrule
		Noise 30 dB & all-15 & $0.399 \pm 0.011$ & $0.660 \pm 0.008$ & $0.354 \pm 0.003$ & $-0.044 \pm 0.004$ \\
		 & disc-7 & $0.287 \pm 0.018$ & $0.274 \pm 0.009$ & $0.497 \pm 0.004$ & $-0.095 \pm 0.007$ \\
		Noise 20 dB & all-15 & $0.138 \pm 0.015$ & $0.631 \pm 0.006$ & $0.460 \pm 0.005$ & $-0.049 \pm 0.005$ \\
		 & disc-7 & $0.112 \pm 0.020$ & $0.212 \pm 0.006$ & $0.562 \pm 0.004$ & $-0.104 \pm 0.012$ \\
		Amplitude $\times 1.05$ & all-15 & $0.704 \pm 0.012$ & $0.762 \pm 0.005$ & $0.126 \pm 0.002$ & $+0.002 \pm 0.004$ \\
		 & disc-7 & $0.658 \pm 0.022$ & $0.492 \pm 0.010$ & $0.115 \pm 0.003$ & $+0.004 \pm 0.008$ \\
		Amplitude $\times 0.95$ & all-15 & $0.720 \pm 0.009$ & $0.772 \pm 0.008$ & $0.108 \pm 0.002$ & $+0.009 \pm 0.004$ \\
		 & disc-7 & $0.698 \pm 0.011$ & $0.514 \pm 0.011$ & $0.090 \pm 0.002$ & $+0.020 \pm 0.008$ \\
		Shift $+4$ samples & all-15 & $0.793 \pm 0.003$ & $0.859 \pm 0.004$ & $0.007 \pm 0.002$ & $-0.009 \pm 0.001$ \\
		 & disc-7 & $0.818 \pm 0.010$ & $0.701 \pm 0.007$ & $0.007 \pm 0.001$ & $-0.019 \pm 0.003$ \\
		Shift $-4$ samples & all-15 & $0.854 \pm 0.009$ & $0.872 \pm 0.005$ & $0.005 \pm 0.000$ & $-0.002 \pm 0.002$ \\
		 & disc-7 & $0.815 \pm 0.016$ & $0.729 \pm 0.010$ & $0.006 \pm 0.001$ & $-0.004 \pm 0.004$ \\
		\bottomrule
	\end{tabular}
	\end{adjustwidth}
\end{table}

Table~\ref{tab:sensitivity} shows a clear ordering across the distortions. After alignment, circular shifts give the smallest distributional changes ($W_1\leq0.007$). Amplitude scaling preserves physical-mask IoU on average, although disc-7 top-$L$ stability is $0.492$--$0.514$, indicating changes in the selected positions. Additive noise has the largest effect. At 20~dB, the disc-7 mean-map correlation falls to $0.112$, coordinatewise $W_1$ increases to $0.562$, and physical-mask IoU decreases by $0.104$. These results measure local sensitivity under the six tested distortions and correspond to the quantities in Corollary~\ref{cor:input}.

\subsection{Finite-Sample Self-Coverage of Explanation Bands}
\label{sec:calibration}

This experiment tests whether explanation bands have the expected finite-sample coverage for held-out draws generated by the same method. For MC Dropout and Laplace, 40 draws per instance are divided into 20 draws for constructing the bands and 20 for evaluating them. Reference coverage is computed from 100 random 20--20 partitions using the same interpolation rule as the implementation. For the deep ensemble, a min--max band is constructed from four members and evaluated on the remaining member in turn. Its leave-one-out (LOO) coverage is compared with the exchangeability reference $3/5$; with ties, a closed band has coverage of at least this value. Exchangeability is assumed from the common generation process rather than from the seed values themselves.

\begin{table}[H]
	
	\footnotesize
	\centering
	\caption{Finite-sample self-coverage of explanation bands. Reference values use the same band construction and exchangeable partitions as the implementation. Difference is empirical coverage minus reference coverage.}
	\label{tab:calibration}
	\begin{tabular}{@{}lcccc@{}}
		\toprule
		Method & Band level & Empirical & Reference & Difference \\
		\midrule
		MC Dropout & 0.50 & 0.455 & 0.455 & $+0.000$ \\
		 & 0.80 & 0.730 & 0.730 & $-0.000$ \\
		 & 0.90 & 0.821 & 0.821 & $-0.001$ \\
		Laplace & 0.50 & 0.446 & 0.470 & $-0.024$ \\
		 & 0.80 & 0.725 & 0.741 & $-0.016$ \\
		 & 0.90 & 0.817 & 0.831 & $-0.014$ \\
		Deep ensemble (LOO min--max) & 0.60 & 0.601 & 0.600 & $+0.001$ \\
		\bottomrule
	\end{tabular}
\end{table}

MC Dropout closely matches its finite-sample reference at all three band levels. Deep-ensemble LOO coverage is also close to its reference ($0.601$ versus $0.600$). Laplace coverage is lower than its reference by $0.014$--$0.024$. These results describe coverage for held-out draws produced by the same method; they do not compare the bands with a true posterior explanation distribution.

\section{Conclusion}
\label{sec:conclusion}

This paper defines explanation distributions for measurable deterministic and stochastic attribution operators and summarises them using the UA-RAO. The theory separates approximation bias from finite-sample error in explanation space, includes a local-stability bound with a boundary-crossing term, separates model variability from stochastic-explainer variability, and gives results for aligned input sensitivity and finite-sample explanation bands.

On the evaluated PQD benchmark, deep-ensemble mean occlusion improves average disc-7 IoU over deterministic occlusion by $0.042$ across paired splits. The improvement depends on the disturbance type and does not occur for interruption or oscillatory transient. Among the tested operators, occlusion gives the highest disc-7 IoU, while LIME gives the highest RMA. Among the UA-RAO summaries, the mean gives the highest IoU. The explanation bands achieve approximately their expected finite-sample coverage for held-out draws generated by the same method. Under the six tested input distortions, the explanations change least after aligned temporal shifts, moderately after amplitude scaling, and most under additive noise.

Future work should test more complex architectures and attribution operators, develop external criteria for evaluating explanation uncertainty, and study a wider range of input perturbations. It also remains to be tested whether the observed patterns hold for larger, heavily overparameterised models with more complex posterior distributions. The same push-forward framework can be used for this purpose. Efficient model-distribution methods such as linearised Laplace approximations \cite{immer2021improving,daxberger2021laplace} may support real-time monitoring.

\bibliographystyle{elsarticle-num}
\bibliography{refer.bib}
\end{document}